\documentclass[aap]{stylefile}
\RequirePackage{amsthm,amsmath,amsfonts,amssymb}
\RequirePackage{natbib}
\RequirePackage{mathtools}
\RequirePackage[shortlabels]{enumitem}
\RequirePackage{algorithm}
\RequirePackage{algorithmic}
\RequirePackage{bm}
\RequirePackage{graphicx}
\RequirePackage{tcolorbox}
\RequirePackage[colorlinks,citecolor=blue,urlcolor=blue]{hyperref}
\DeclarePairedDelimiter{\ceil}{\lceil}{\rceil}

\startlocaldefs
\theoremstyle{plain}
\newtheorem{theorem}{Theorem}[section]
\newtheorem{lemma}{Lemma}[section]

\newtheorem{proposition}{Proposition}[section]

\theoremstyle{definition}

\newtheorem{example}{Example}[section]
\newtheorem{assumption}{Assumption}[section]

\theoremstyle{remark}
\newtheorem*{remark}{Remark}

\endlocaldefs

\begin{document}

\begin{frontmatter}
\title{Target Network and Truncation Overcome the\\ Deadly Triad in $Q$-learning}

\begin{aug}
\author[A]{\fnms{Zaiwei} \snm{Chen}\ead[label=e1,mark]{zchen458@gatech.edu}},
\author[B]{\fnms{John-Paul} \snm{Clarke}\ead[label=e2,mark]{johnpaul@utexas.edu}},
\and
\author[A]{\fnms{Siva Theja} \snm{Maguluri}\ead[label=e3,mark]{siva.theja@gatech.edu}}

\address[A]{
Georgia Institute of Technology,
\printead{e1,e3}}

\address[B]{
The University of Texas at Austin,
\printead{e2}}

\end{aug}

\begin{abstract}
$Q$-learning with function approximation is one of the most empirically successful while theoretically mysterious reinforcement learning (RL) algorithms, and was identified in \cite{sutton1999open} as one of the most important theoretical open problems in the RL community. Even in the basic linear function approximation setting, there are well-known divergent examples. In this work, we show that \textit{target network} and \textit{truncation} together are enough to provably stabilize $Q$-learning with linear function approximation, and we establish the finite-sample guarantees. The result implies an $\Tilde{\mathcal{O}}(\epsilon^{-2})$ sample complexity up to a function approximation error. Moreover, our results do not require strong assumptions or modifying the problem parameters as in existing literature.
\end{abstract}

\end{frontmatter}

\section{Introduction}
The Deep $Q$-Network \citep{mnih2015human}, as a typical example of $Q$-learning with function approximation, is one of the most successful algorithms to solve the reinforcement learning (RL) problem, and hence is viewed as a milestone in the development of modern RL. On the other hand, the behavior of $Q$-learning with function approximation is theoretically not well understood, and was identified in \cite{sutton1999open} as one of four most important theoretical open problems. In fact, the infamous deadly triad \citep{sutton2015introduction} is present in $Q$-learning with function approximation, and hence even in the basic setting where linear function approximation is used, the algorithm was shown to be unstable in general \citep{baird1995residual}. 

While theoretically unclear, it was empirically evident from \cite{mnih2015human}  that the following three ingredients: \textit{experience replay}, \textit{target network}, and \textit{truncation} together overcome the divergence of $Q$-learning with function approximation. In this work, we focus on $Q$-learning with linear function approximation for infinite horizon discounted Markov decision processes (MDPs), and show theoretically that target network together with truncation is sufficient to provably stabilize $Q$-learning. The main contributions of this work are summarized in the following.

\begin{itemize}
    \item \textbf{Finite-Sample Guarantees.} 
    We establish finite-sample guarantees of the output of $Q$-learning with target network and truncation to the optimal $Q$-function $Q^*$ up to a function approximation error. This is the first variant of $Q$-learning with linear function approximation that is provably stable (without needing strong assumptions), and uses a single trajectory of Markovian samples. The result implies an $\Tilde{\mathcal{O}}(\epsilon^{-2})$ sample complexity, which matches with the sample complexity of $Q$-learning in the tabular setting, and is known to be optimal up to a logarithmic factor. The function approximation error in our finite-sample bound well captures the approximation power of the chosen function class. In the special case of tabular setting, or assuming the function class is closed under Bellman operator, our result implies asymptotic convergence in the mean-square sense to the optimal $Q$-function $Q^*$. 
    \item \textbf{Broad Applicability.} In existing literature, to stabilize $Q$-learning with linear function approximation, one usually requires strong assumptions on the underlying MDP and/or the approximating function class. Those assumptions include but not limited to the function class being complete with respect to the Bellman operator, the MDP being linear (or close to linear), and a so-called strong negative drift assumption, etc. In this work, we do not require any of those assumptions. Specifically, our result holds as long as the policy used to collect samples enables the agent to sufficiently explore the state-action space, which is to some extent a necessary requirement to find an optimal policy in RL.
\end{itemize}

\subsection{Related Work}\label{subsec:literature}
The $Q$-learning algorithm was first proposed in \cite{watkins1992q}. Since then, theoretically understanding the behavior of $Q$-learning has been a major topic in the RL community. In particular, the asymptotic convergence of $Q$-learning was established in \cite{tsitsiklis1994asynchronous,jaakkola1994convergence,borkar2000ode,lee2019unified}, and the asymptotic convergence rate in \cite{szepesvari1998asymptotic,devraj2017zap}. Beyond the asymptotic behavior, recently there has been an increasing interest in studying finite-sample convergence guarantees of $Q$-learning. Here is a non-exhaustive list: \cite{even2003learning,beck2012error,beck2013improved,chen2020finite,chen2021finite,chandak2021concentration,borkar2021concentration,li2020sample,li2021q,li2021breaking,wainwright2019stochastic,wainwright2019variance,jin2018q,qu2020finite}, leading to the optimal $\Tilde{\mathcal{O}}(\epsilon^{-2})$ sample complexity of $Q$-learning. Other variants of $Q$-learning such as zap $Q$-learning and double $Q$-learning were proposed and studied in \cite{devraj2017zap} and \cite{hasselt2010double}, respectively.

When using function approximation, the infamous deadly triad  (i.e., function approximation, off-policy sampling, and bootstrapping) \citep{sutton2018reinforcement} appears in $Q$-learning, and the algorithm can be unstable even when linear function approximation is used. This is evident from the divergent MDP example constructed in \cite{baird1995residual}. There are many attempts to stabilize $Q$-learning with linear function approximation, which are summarised in the following. See Appendix \ref{ap:literature} for a more detailed survey about existing results and their limitations.

\textbf{Strong Negative Drift Assumption.} The asymptotic convergence of $Q$-learning with linear function approximation was established in \cite{melo2008analysis} under a ``negative drift'' assumption. Under similar assumptions, the finite-sample analysis of $Q$-learning, as well as its on-policy variant SARSA, was performed in \cite{chen2019finitesample,gao2021finite,lee2019unified,zou2019finite} for using linear function approximation, and in \cite{xu2020finite,cai2019neural} for using neural network approximation. However, such negative drift assumption is highly artificial, highly restrictive, and is impossible to satisfy unless the discount factor of the MDP is extremely small. See Appendix \ref{ap:negative_drift} for a more detailed explanation. In this work, we do not require such negative drift assumption or any of its variants to stabilize $Q$-learning with linear function approximation.

\textbf{Modifying the Problem Discount Factor.}
Very recently, new convergent variants of $Q$-learning with linear function approximation were proposed in \cite{carvalho2020new,zhang2021breaking}, where target network was used in the algorithm. However, as we will see later in Section \ref{subsec:insufficiency_of_TN}, target network alone is not sufficient to provably stabilize $Q$-learning. The reason that \cite{carvalho2020new,zhang2021breaking} achieve convergence of $Q$-learning is by implicitly modifying the discount factor. In fact, the problem they are effectively solving is no longer the original MDP, but an MDP with a much smaller discount factor, which is the reason why their algorithms do not converge to the optimal $Q$-function $Q^*$ in the tabular setting. See Appendices \ref{ap:TN} and \ref{ap:Melo} for more details. In this work we do not modify the original problem parameters to achieve stability, and in the special case of tabular RL, we have convergence to $Q^*$.

\textbf{The Greedy-GQ Algorithm.} A two time-scale variant of $Q$-learning with linear function approximation, known as Greedy-GQ, was proposed in \cite{maei2010toward}. The algorithm is designed based on minimizing the projected Bellman error using stochastic gradient descent. Although the Greedy-GQ algorithm is stable without needing the negative drift assumption, since the Bellman error is in general non-convex, Greedy-GQ algorithm can only guarantee convergence to stationary points. As a result, there are no performance guarantees on how well the limit point approximates the optimal $Q$-function $Q^*$. Although finite-sample bounds for Greedy-GQ were recently established in \cite{wang2020finite,ma2021greedy,xu2021sample}, due to the lack of global optimality, the finite-sample bounds were on the gradient of the Bellman error rather than the distance to $Q^*$. In this work we provide finite-sample guarantees to the optimal $Q$-function $Q^*$ (up to a function approximation error).

\textbf{Fitted $Q$-Iteration and Its Variants.} Fitted $Q$-iteration is proposed in \cite{ernst2005tree} as an offline variant of $Q$-learning. The finite-sample guarantees of fitted $Q$-iteration (or more generally fitted value iteration) were established in \cite{szepesvari2005finite, munos2008finite}. More recently, \cite{xie2020batch} proposes a variant of batch RL algorithms called BVFT, where the authors establish an $\Tilde{\mathcal{O}}(\epsilon^{-4})$ sample complexity under the realizability assumption. Notably, \cite{szepesvari2005finite, munos2008finite} employed truncation technique to ensure the boundedness of the function approximation class. Such truncation technique dates back to \cite{gyorfi2002distribution}. We use the same truncation technique in this paper.
In the special case of linear function approximation, $Q$-learning with target network can be viewed as an approximate way of implementing the fitted $Q$-iteration, where stochastic gradient descent was used as a way of performing such fitting. Compared to \cite{szepesvari2005finite, munos2008finite}, the main difference of this work is that our algorithm is implemented in an online manner, and is driven by a single trajectory of Markovian samples. 

Another variant of fitted $Q$-iteration targeting finite horizon MDPs was proposed in \cite{du2019provably} using a distribution shift checking oracle. However, \cite{du2019provably} requires the approximating function class to contain the optimal $Q$-function, and only polynomial sample complexity, i.e., $\Tilde{\mathcal{O}}(\epsilon^{-n})$ for some positive integer $n$, was established.  In this work, we do not require $Q^*$ to be within our chosen function class, and our algorithm achieves the optimal $\Tilde{\mathcal{O}}(\epsilon^{-2})$ sample complexity.

\textbf{Linear MDP Model.} In the special case that the MDP has linear transition dynamics and linear reward, convergent variants of $Q$-learning with linear function approximation were designed and analyzed in \cite{yang2019sample,yang2020reinforcement,jin2020provably,zhou2021provably,he2021uniform,li2021sampleefficient}. Such linear model assumption can be relaxed to the case where the MDP is approximately linear. In this work, we do not make any assumption on the underlying structure of the MDP, except the uniform ergodicity of the Markov chain induced by the behavior policy.

\textbf{Other Work.} \cite{du2020agnostic} studies $Q$-learning with function approximation for deterministic MDPs. The Deep $Q$-Network was studied in \cite{fan2020theoretical}. See Appendix \ref{ap:DQN} for a more detailed discussion about the Deep $Q$-Network.

\section{Background on RL and $Q$-Learning}\label{sec:background}
We model the RL problem as an infinite horizon discounted MDP defined by a $5$-tuple $(\mathcal{S},\mathcal{A},\mathcal{P},\mathcal{R},\gamma)$, where $\mathcal{S}$ is a finite set of states, $\mathcal{A}$ is a finite set of actions, $\mathcal{P}=\{P_a\in\mathbb{R}^{|\mathcal{S}|\times|\mathcal{S}|}\mid a\in\mathcal{A}\}$ is a set of \textit{unknown} transition probability matrices, $\mathcal{R}:\mathcal{S}\times\mathcal{A}\mapsto[0,1]$ is an \textit{unknown} reward function, and $\gamma\in (0,1)$ is the discount factor. Our results can be generalized to continuous-state finite-action MDPs. We restrict our attention to finite-state setting for ease of exposition.

Define the state-action value function of a policy $\pi$ by $Q^\pi(s,a)=\mathbb{E}[\sum_{k=0}^\infty\gamma^k\mathcal{R}(S_k,A_k)\mid S_0=s,A_0=a]$ for all $(s,a)$. The goal is to find an optimal policy $\pi^*$ so that its associated $Q$-function (denoted by $Q^*$) is maximized uniformly for all $(s,a)$. A well-known relation between the optimal $Q$-function and any optimal policy $\pi^*$ states that $\pi^*(\cdot|s)$ is supported on the set $\arg\max_{a\in\mathcal{A}}Q^*(s,a)$ for all $s$. Therefore, to find an optimal policy, it is enough the find the optimal $Q$-function, which is the motivation of the $Q$-learning algorithm. The $Q$-learning algorithm is designed to find $Q^*$ by solving the Bellman equation $Q^*=\mathcal{H}(Q^*)$ using stochastic approximation, and provably converges. However, $Q$-learning becomes intractable for MDPs with large state-action space. This motivates the use of function approximation, where the idea is to approximate the optimal $Q$-function from a pre-specified function class. 

In this work, we focus on using linear function approximation. Let $\phi_i\in\mathbb{R}^{|\mathcal{S}||\mathcal{A}|}$, $i=1,2,...,d$, be a set of basis vectors, and denote $\phi(s,a)=(\phi_1(s,a),\cdots,\phi_d(s,a))\in\mathbb{R}^d$ for all $(s,a)$. We assume without loss of generality that the basis vectors $\{\phi_i\}_{1\leq i\leq d}$ are linearly independent, and are normalized so that $\|\phi(s,a)\|_1\leq 1$ for all $(s,a)$, where $\|\cdot\|_1$ stands for the $\ell_1$-norm. Let $\Phi\in\mathbb{R}^{|\mathcal{S}||\mathcal{A}|\times d}$ be defined by
\begin{align*}
	\Phi = \left[\begin{array}{ccc}
	\vert     &  & \vert\\
	\phi_1     & ... & \phi_d\\
	\vert & & \vert
	\end{array}\right] = \left[\begin{array}{ccc}
	\mbox{---}   & \phi(s_{1},a_{1})^\top & \mbox{---}\\
	...    & ... & ...\\
	\mbox{---}   & \phi(s_{|\mathcal{S}|},a_{|\mathcal{A}|})^\top & \mbox{---}
	\end{array}\right].
\end{align*}
Using the feature matrix $\Phi$, the linear sub-space spanned by $\{\phi_i\}_{1\leq i\leq d}$ can be compactly written as $\mathcal{W}=\{Q_\theta\in\mathbb{R}^{|\mathcal{S}||\mathcal{A}| }\mid Q_\theta=\Phi\theta,\;\theta\in \mathbb{R}^d\}$. The goal of $Q$-learning with linear function approximation is to design a stable algorithm that provably finds an approximation of the optimal $Q$-function $Q^*$ from the linear sub-space $\mathcal{W}$.

\section{Algorithm and Finite-Sample Guarantees}\label{sec:finite-sample-bounds}
In this section, we first present the algorithm of $Q$-learning with linear function approximation using target network and truncation. Then we provide the finite-sample guarantees of our algorithm to the optimal $Q$-function $Q^*$ up to a function approximation error. The detailed proofs of all technical results are provided in the Appendix.

\subsection{Stable Algorithm Design}
To present our algorithm, we introduce the truncation operator $\ceil{\cdot}$ in the following. For any vector $x$, let $\ceil{x}$ be the resulting vector of $x$ component-wisely truncated from both above and below at $r=1/(1-\gamma)$, i.e., for each component $\ceil{x}_i$ of $\ceil{x}$, we have $\ceil{x}_i=r$ if $x_i>r$, $\ceil{x}_i=x_i$ if $x_i\in [-r,r]$, and $\ceil{x}_i=-r$ if $x_i<-r$. The reason that we pick the truncation level $r$ to be $1/(1-\gamma)$ is that $\|Q^*\|_\infty\leq 1/(1-\gamma)$. Therefore by performing truncation we do not exclude $Q^*$.

\begin{algorithm}[h]\caption{$Q$-Learning with Linear Function Approximation: Target Network and Truncation}\label{alg:main}
\begin{algorithmic}[1] 
	\STATE {\bfseries Input:} Integers $T$, $K$, initializations $\theta_{t,0}=\bm{0}$ for all $t=0,1,...,T-1$ and $\hat{\theta}_0=\bm{0}$, behavior policy $\pi_b$
	\FOR{$t=0,1,\cdots,T-1$}
	\FOR{$k=0,1,\cdots,K-1$}
	\STATE Sample $A_k\sim \pi_b(\cdot|S_k)$, $S_{k+1}\sim P_{A_k}(S_k,\cdot)$
	\STATE $\theta_{t,k+1}=\theta_{t,k}+\alpha_k\phi(S_k,A_k)(\mathcal{R}(S_k,A_k)+\gamma\max_{a'\in\mathcal{A}}\ceil{\phi(S_{k+1},a')^\top\hat{\theta}_t}-\phi(S_k,A_k)^\top \theta_{t,k})$
	\ENDFOR
	\STATE $\hat{\theta}_{t+1}=\theta_{t,K}$
	\STATE $S_0=S_K$
	\ENDFOR
	\STATE\textbf{Output:} $\hat{\theta}_T$
\end{algorithmic}
\end{algorithm}

Several remarks are in order. First of all, Algorithm \ref{alg:main} is simple, easy to implement, and can be generalized to using arbitrary parametric function approximation in a straightforward manner (see Appendix \ref{ap:DQN}). Second, in addition to $\{\theta_{t,k}\}$, we introduce $\{\hat{\theta}_t\}$ as the target network parameter, which is fixed in the inner loop where we update $\theta_{t,k}$, and is synchronized to the last iterate $\theta_{t,K}$ in the outer loop. Target network was first introduced in \cite{mnih2015human} for the design of the celebrated Deep $Q$-Network. Finally, before using the $Q$-function estimate associated with the target network in the inner-loop, we first truncate it at level $r$ (see line 5 of Algorithm \ref{alg:main}). 

Note that the location where we impose the truncation operator is different from that in the Deep $Q$-Network \citep{mnih2015human}, where instead of only truncating $\phi(S_{k+1},a')^\top\hat{\theta}_t$, truncation is performed for the entire temporal difference $\mathcal{R}(S_k,A_k)+\gamma\max_{a'\in\mathcal{A}}\phi(S_{k+1},a')^\top\hat{\theta}_t-\phi(S_k,A_k)^\top \theta_{t,k}$. Similar truncation technique has been employed in \cite{munos2008finite,jin2018q}. The reason that target network and truncation together ensure the stability of $Q$-learning with linear function approximation will be illustrated in detail in Section \ref{sec:roadmap}.

On the practical side, Algorithm \ref{alg:main} uses a \textit{single trajectory} of Markovian samples generated by the behavior policy $\pi_b$ (see line 4 and line 8 of Algorithm \ref{alg:main}). Therefore, the agent does not have to constantly reset the system. Our result can be easily generalized to the case where one uses time-varying behavior policy (i.e., the behavior policy is updated across the iterations of the target network) as long as it ensures sufficient exploration. For example, one can use the $\epsilon$-greedy policy or the Boltzmann exploration policy (aka. softmax policy) with respect to the $Q$-function estimate associated with the target network $Q_{\hat{\theta}_t}$ as the behavior policy.

\subsection{Finite-Sample Guarantees}
To present the finite-sample guarantees of Algorithm \ref{alg:main}, we first formally state our assumption about the behavior policy $\pi_b$ and introduce necessary notation.

\begin{assumption}\label{as:MC}
The behavior policy $\pi_b$ satisfies $\pi_b(a|s)>0$ for all $(s,a)$, and induces an irreducible and aperiodic Markov chain $\{S_k\}$.
\end{assumption}

This assumption ensures that the behavior policy sufficient explores the state-action space, and is commonly imposed for value-based RL algorithms in the literature \cite{tsitsiklis1997analysis}. Note that Assumption \ref{as:MC} implies that the Markov chain $\{S_k\}$ admits a unique stationary distribution, denoted by $\mu\in\Delta^{|\mathcal{S}|}$, and mixes at a geometric rate \citep{levin2017markov}. As a result, letting $t_\delta=\min\{k\geq 0\;:\;\max_{s\in\mathcal{S}}\|P_{\pi_b}^k(s,\cdot)-\mu(\cdot)\|_{\text{TV}}\leq \delta\}$ be the mixing time of the Markov chain $\{S_k\}$ (induced by $\pi_b$) with precision $\delta>0$, then under Assumption \ref{as:MC} we have $t_\delta=\mathcal{O}(\log(1/\delta))$. 

Under Assumption \ref{as:MC}, the Markov chain $\{(S_k,A_k)\}$ also has a unique stationary distribution.
Let $D\in \mathbb{R}^{|\mathcal{S}||\mathcal{A}|\times|\mathcal{S}||\mathcal{A}|}$ be a diagonal matrix with the unique stationary distribution of $\{(S_k,A_k)\}$ on its diagonal, i.e., $D((s,a),(s,a))=\mu(s)\pi_b(a|s)$ for all $(s,a)$. Moreover, let a norm $\|\cdot\|_D$ be defined by $\|x\|_D=(x^\top Dx)^{1/2}$. Denote $\lambda_{\min}$ as the minimum eigenvalue of the positive definite matrix $\Phi^\top D\Phi$. 

Let $\text{Proj}_{\mathcal{W}}(\cdot)$ be the projection operator onto the linear sub-space $\mathcal{W}$ with respect to the weighted $\ell_2$-norm $\|\cdot\|_D$. Note that $\text{Proj}_{\mathcal{W}}(\cdot)$ is explicitly given by $\text{Proj}_{\mathcal{W}}(Q)=\Phi(\Phi^\top D\Phi)^{-1}\Phi^\top DQ$ for any $Q\in\mathbb{R}^{|\mathcal{S}||\mathcal{A}|}$. Let $\mathcal{E}_{\text{approx}}:=\sup_{Q\in\mathcal{W}:\|Q\|_\infty\leq r}\|\ceil{\text{Proj}_{\mathcal{W}}\mathcal{H}(Q)}-\mathcal{H}(Q)\|_\infty$, which captures the approximation power of the chosen function class. Denote $\hat{Q}_t=\ceil{\Phi \hat{\theta}_t}$ as the truncated $Q$-function estimate associated with the target network $\hat{\theta}_t$. 

We next present the finite-sample bounds. For ease of exposition, we only present the case where we use constant stepsize in the inner-loop of Algorithm \ref{alg:main}, i.e., $\alpha_k\equiv \alpha$. The results for using various diminishing stepsizes are straightforward extensions \cite{chen2019finitesample}.

\begin{theorem}\label{thm:main}
Consider $\hat{\theta}_T$ of Algorithm \ref{alg:main}. Suppose that Assumption \ref{as:MC} is satisfied, the constant stepsize $\alpha$ is chosen such that $\alpha\leq \frac{\lambda_{\min}(1-\gamma)^2}{130}$, and $K\geq t_\alpha+1$. Then we have for any $T\geq 0$ that
\begin{align}
    \mathbb{E}[\|\hat{Q}_T-Q^*\|_\infty]
    \leq\;& \underbrace{\gamma^T\|\hat{Q}_0-Q^*\|_\infty}_{E_1: \text{ Error due to fixed-point iteration}}+\underbrace{\frac{2(1-\lambda_{\min}\alpha)^{\frac{K-t_\alpha-1}{2}}}{\lambda_{\min}^{1/2}(1-\gamma)^2}}_{E_2:\text{ Bias in the inner-loop}}\nonumber\\
    &+\underbrace{\frac{24\sqrt{\alpha (t_\alpha+1)}}{\lambda_{\min}(1-\gamma)^2}}_{E_3: \text{ Variance in the inner-loop}}+\underbrace{\frac{\mathcal{E}_{\text{approx}}}{1-\gamma}.}_{E_4: \text{ Function approximation error}}\label{eq:finite-sample-bounds}
\end{align}
As a result, to obtain $\mathbb{E}[\|\hat{Q}_T-Q^*\|_\infty]\leq \epsilon+\frac{\mathcal{E}_{\text{approx}}}{1-\gamma}$ for a given accuracy $\epsilon$, the sample complexity is
\begin{align*}
    \mathcal{O}\left(\frac{\log^2(1/\epsilon)}{\epsilon^2}\right)\Tilde{\mathcal{O}}\left(\frac{1}{(1-\gamma)^4}\right).
\end{align*}
\end{theorem}
\begin{remark}
While commonly used in existing literature studying RL with function approximation, it was argued in \cite{khodadadian2021finite2} that sample complexity is strictly speaking not well-defined when the asymptotic error is non-zero. Here we present the ``sample complexity'' in the same sense as in existing literature to enable a fair comparison.

\end{remark}
Theorem \ref{thm:main} is by far the strongest result of $Q$-learning with linear function approximation in the literature in that it achieves the optimal $\Tilde{\mathcal{O}}(\epsilon^{-2})$ sample complexity without needing strong assumptions or modifying the problem parameters.

In our finite-sample bound, the term $E_1$ goes to zero geometrically fast as $T$ goes to infinity. In fact, the term $E_1$ captures the error due to fixed-point iteration. That is, if we had a complete basis (hence no function approximation error), and were able to perform value iteration to solve the Bellman equation $Q^*=\mathcal{H}(Q^*)$ (hence no stochastic error), $E_1$ is the only error term. 

The terms $E_2$ and $E_3$ represent the bias and variance in the inner-loop of Algorithm \ref{alg:main}. Since the target network parameter $\hat{\theta}_t$ is fixed in the inner-loop, the update equation in Algorithm \ref{alg:main} line 5 can be viewed as a linear stochastic approximation algorithm under Markovian noise.  When using constant stepsize, the bias goes to zero geometrically fast as $K$ goes to infinity but the variance is a constant proportional to $\sqrt{\alpha t_\alpha}$. Since geometric mixing implies $t_\alpha=\mathcal{O}(\log(1/\alpha))$, the term $\sqrt{\alpha t_\alpha}$ can be made arbitrarily small by using small enough constant stepsize. This agrees with existing literature studying linear stochastic approximation \cite{srikant2019finite}. When using diminishing stepsizes with a suitable decay rate, one can easily show using results in \cite{chen2019finitesample} that both $E_1$ and $E_2$ go to zero at a rate of $\mathcal{O}(1/\sqrt{K})$, therefore the resulting sample complexity is the same as when using constant stepsize. 

The term $E_4$ captures the error due to using function approximation. Recall that we define  $\mathcal{E}_{\text{approx}}=\sup_{Q\in\mathcal{W}:\|Q\|_\infty\leq r}\|\ceil{\text{Proj}_{\mathcal{W}}\mathcal{H}(Q)}-\mathcal{H}(Q)\|_\infty$. Therefore to make the function approximation error small, one only needs to approximate the functions that are one-step reachable under the Bellman operator. In addition, using truncation also helps reducing the function approximation error to some extend since $\|\ceil{\text{Proj}_{\mathcal{W}}\mathcal{H}(Q)}-\mathcal{H}(Q)\|_\infty\leq \|\text{Proj}_{\mathcal{W}}\mathcal{H}(Q)-\mathcal{H}(Q)\|_\infty$ for any $Q$ such that $\|Q\|_\infty\leq r$. The $1/(1-\gamma)$ factor in $E_4$ also appears in TD-learning with linear function approximation \cite{tsitsiklis1997analysis}, where it was shown to be not removable in general. Observe that $E_4$ vanishes (and hence we have convergence to $Q^*$) when (1)  we are in the tabular setting, or (2) we use a complete basis (i.e., $\Phi$ being an invertible matrix), or (3) under the completeness assumption in existing literature, which requires $\mathcal{H}(Q)\in\mathcal{W}$ whenever $Q\in\mathcal{W}$. In existing work \cite{carvalho2020new,zhang2021breaking}, the algorithm does not converge to $Q^*$ even in the tabular setting (see Appendix \ref{ap:literature}).

\section{The reason that Target Network and Truncation Stabilize $Q$-Learning}\label{sec:roadmap}
In the previous section, we presented the algorithm and the finite-sample guarantees. In this section, we elaborate in detail why target network and truncation together are enough to stabilize $Q$-learning.

\textbf{Summary.} We start with the classical semi-gradient $Q$-learning with linear function approximation algorithm in Section \ref{subsec:classical_Q}, which unfortunately is not necessarily stable, as evidenced by the divergent counter-example constructed in \cite{baird1995residual}. In Section \ref{subsec:TN}, We show that by adding target network to $Q$-learning, the resulting algorithm successfully overcomes the divergence in the MDP example in \cite{baird1995residual}. However, beyond the example in \cite{baird1995residual}, target network alone is not sufficient to stabilize $Q$-learning. In fact, we show in Section \ref{subsec:insufficiency_of_TN} that $Q$-learning with target network diverges for another MDP example constructed in \cite{chen2019finitesample}. In Section \ref{subsec:TC}, we show that by further adding truncation, the resulting algorithm (i.e., Algorithm \ref{alg:main}) is provably stable and achieves the optimal $\Tilde{\mathcal{O}}(\epsilon^{-2})$ sample complexity. The reason that truncation successfully stabilizes $Q$-learning is due to an insightful observation regarding the relation between truncation and projection.

\subsection{Classical Semi-Gradient $Q$-Learning}\label{subsec:classical_Q}
We begin with the classical semi-gradient $Q$-learning with linear function approximation algorithm \citep{bertsekas1996neuro,sutton2018reinforcement}. With a trajectory of samples $\{(S_k,A_k)\}$ collected under the behavior policy $\pi_b$ and an arbitrary initialization $\theta_0$, the semi-gradient $Q$-learning algorithm updates the parameter $\theta_k$  according to the following formula:

\begin{align}
    \theta_{k+1}=\theta_k+\alpha_k\phi(S_k,A_k)\left(\mathcal{R}(S_k,A_k)+\gamma\max_{a'\in\mathcal{A}}\phi(S_{k+1},a')^\top \theta_k-\phi(S_k,A_k)^\top \theta_k\right).\label{alg:semi_gradient_QLFA}
\end{align}
The reason that update (\ref{alg:semi_gradient_QLFA}) is called semi-gradient $Q$-learning is that it can be interpreted as a one step stochastic semi-gradient descent for minimizing the Bellman error. See \cite{bertsekas1996neuro} for more details. Unfortunately, Algorithm (\ref{alg:semi_gradient_QLFA}) does not necessarily converge, as evidenced by the divergent example provided in \cite{baird1995residual}. The MDP example contructed in \cite{baird1995residual} has $7$ states and $2$ actions. To perform linear function approximation, $14$ linearly independent basis vectors are chosen. See Appendix \ref{ap:Baird} for more details about this MDP. The important thing to notice about this example is that the number of basis vectors is equal to the size of the state-action space, i.e., $d=|\mathcal{S}||\mathcal{A}|$. Hence rather than doing function approximation, we are essentially doing a change of basis. Surprisingly even in this setting, Algorithm (\ref{alg:semi_gradient_QLFA}) diverges. Due to the divergence nature, \cite{melo2008analysis,chen2019finitesample,lee2019unified} impose strong negative drift assumptions to ensure the stability of Algorithm (\ref{alg:target_network_QLFA}).

By viewing Algorithm (\ref{alg:semi_gradient_QLFA}) as a stochastic approximation algorithm, the target equation Algorithm (\ref{alg:semi_gradient_QLFA}) is trying to solve is $\mathbb{E}_{S_k\sim\mu,A_k\sim\pi_b(\cdot|S_k)}[\phi(S_k,A_k)(\mathcal{R}(S_k,A_k)+\gamma\max_{a'\in\mathcal{A}}\phi(S_{k+1},a')^\top \theta-\phi(S_k,A_k)^\top \theta)]=0$.
The previous equation can be written compactly using the Bellman optimality operator $\mathcal{H}(\cdot)$ and the diagonal matrix $D$ as
\begin{align}\label{eq:30}
    \Phi^\top D(\mathcal{H}(\Phi\theta)-\Phi \theta)=0,
\end{align}
and is further equivalent to the fixed-point equation
\begin{align}\label{eq:31}
    \theta=\mathcal{H}_{\Phi}(\theta),
\end{align}
where the operator $\mathcal{H}_\Phi:\mathbb{R}^d\mapsto\mathbb{R}^d$ is defined by $\mathcal{H}_{\Phi}(\theta)=(\Phi^\top D\Phi)^{-1}\Phi^\top D\mathcal{H}(\Phi \theta)$. Eq. (\ref{eq:31}) is closely related to the so-called projected Bellman equation. To see this, since $\Phi$ is assumed to have linearly independent columns, Eq. (\ref{eq:31}) is equivalent to 
\begin{align}
    \Phi\theta=\Phi(\Phi^\top D\Phi)^{-1}\Phi^\top D\mathcal{H}(\Phi \theta)=\text{Proj}_{\mathcal{W}}\mathcal{H}(\Phi\theta),\label{eq:PBE}
\end{align}
where $\text{Proj}_{\mathcal{W}}$ denotes the projection operator onto the linear subspace $\mathcal{W}$ (which is spanned by the columns of $\Phi$) with respect to the weighted $\ell_2$-norm $\|\cdot\|_D$. 

We next show that in the complete basis setting, i.e., $d=|\mathcal{S}||\mathcal{A}|$, which covers the Baird's counter-example as a special case, the operator $\mathcal{H}_\Phi(\cdot)$ is in fact a contraction mapping with $\theta^*=\Phi^{-1}Q^*$ being its unique fixed-point. This implies that the design of the classical semi-gradient $Q$-learning algorithm (\ref{alg:semi_gradient_QLFA}) is flawed because \textit{if it were designed as a stochastic approximation algorithm which is in effect performing fixed-point iteration to solve Eq. (\ref{eq:31}), it would converge}. Instead, it was designed as a stochastic approximation algorithm based on Eq. (\ref{eq:30}). While Eq. (\ref{eq:30}) is equivalent to Eq. (\ref{eq:31}), their corresponding stochastic approximation algorithms have different behavior in terms of their convergence or divergence.

To show the contraction property of $\mathcal{H}_\Phi(\cdot)$, first observe that in the complete basis setting we have $\mathcal{H}_{\Phi}(\theta)=(\Phi^\top D\Phi)^{-1}\Phi^\top D\mathcal{H}(\Phi \theta)=\Phi^{-1}\mathcal{H}(\Phi\theta)$. Let $\|\cdot\|_{\Phi,\infty}$ be a norm on $\mathbb{R}^d$ defined by $\|\theta\|_{\Phi,\infty}=\|\Phi \theta\|_\infty$ (the fact that it is indeed a norm can be easily verified). Then we have
\begin{align*}
    \|\mathcal{H}_\Phi(\theta_1)-\mathcal{H}_\Phi(\theta_2)\|_{\Phi,\infty}=\|\mathcal{H}(\Phi\theta_1)-\mathcal{H}(\Phi\theta_2)\|_\infty
    \leq \gamma \|\Phi(\theta_1-\theta_2)\|_\infty=\gamma \|\theta_1-\theta_2\|_{\Phi,\infty}
\end{align*}
for all $\theta_1,\theta_2\in\mathbb{R}^d$,
where the inequality follows from the Bellman optimality operator $\mathcal{H}(\cdot)$ being an $\ell_\infty$-norm contraction mapping.
It follows that the operator $\mathcal{H}_\Phi(\cdot)$ is a contraction mapping with respect to $\|\cdot\|_{\Phi,\infty}$. Moreover, since $\mathcal{H}_\Phi(\theta^*)=\Phi^{-1}\mathcal{H}(\Phi\theta^*)=\Phi^{-1}\mathcal{H}(Q^*)=\Phi^{-1}Q^*=\theta^*$, the point $\theta^*$ is the unique fixed-point of the operator $\mathcal{H}_\Phi(\cdot)$. The previous analysis suggests that we should aim at designing $Q$-learning with linear function approximation algorithm as a fixed-point iteration (implemented in a stochastic manner due to sampling in RL) to solve Eq. (\ref{eq:31}). The resulting algorithm would at least converge for the Baird's MDP example.

\subsection{Introducing Target Network}\label{subsec:TN}
We begin with the following fixed-point iteration for solving the fixed-point equation (\ref{eq:31}):
\begin{align}\label{eq:fi2}
    \theta_{k+1}=(\Phi^\top D\Phi)^{-1}\Phi^\top D\mathcal{H}(\Phi\theta_k),
\end{align}
where we write $\mathcal{H}_\Phi(\cdot)$ explicitly in terms of $\Phi$, $D$, and $\mathcal{H}(\cdot)$. Update (\ref{eq:fi2}) is what we would like to perform if we had complete information about the dynamics of the underlying MDP.
The question is that if there is a stochastic variant of such fixed-point iteration that can be actually implemented in the RL setting where the transition probabilities and the stationary distribution are unknown. The answer is $Q$-learning with target network.

\begin{algorithm}[h]\caption{$Q$-Learning with Linear Function Approximation: Target Network and No Truncation}\label{alg:target_network_QLFA}
\begin{algorithmic}[1] 
	\STATE {\bfseries Input:} Integers $T$, $K$, initializations $\theta_{t,0}=\bm{0}$ for all $t=0,1,...,T-1$ and $\hat{\theta}_0=\bm{0}$, behavior policy $\pi_b$
	\FOR{$t=0,1,\cdots,T-1$}
	\FOR{$k=0,1,\cdots,K-1$}
	\STATE Sample $A_k\sim \pi_b(\cdot|S_k)$, $S_{k+1}\sim P_{A_k}(S_k,\cdot)$
	\STATE $\theta_{t,k+1}=\theta_{t,k}+\alpha_k\phi(S_k,A_k)(\mathcal{R}(S_k,A_k)+\gamma\max_{a'\in\mathcal{A}}\phi(S_{k+1},a')^\top \hat{\theta}_t-\phi(S_k,A_k)^\top \theta_{t,k})$
	\ENDFOR
	\STATE $\hat{\theta}_{t+1}=\theta_{t,K}$
	\STATE $S_0=S_K$
	\ENDFOR
	\STATE\textbf{Output:} $\hat{\theta}_T$
\end{algorithmic}
\end{algorithm}

We next elaborate on why Algorithm \ref{alg:target_network_QLFA} can be viewed as a stochastic variant of the fixed-point iteration (\ref{eq:fi2}). Consider the update equation (line 5) in the inner-loop of Algorithm \ref{alg:target_network_QLFA}. Since the target network is fixed in the inner-loop, the update equation in terms of $\theta_{t,k}$ is in fact a linear stochastic approximation algorithm for solving the following linear system of equations: 
\begin{align}\label{eq:linear}
-\Phi^\top D\Phi\theta+\Phi^\top D\mathcal{H}(\Phi \hat{\theta}_t)=0.
\end{align}
Since the matrix $-\Phi^\top D\Phi$ is negative definite, the asymptotic convergence of the inner-loop update follows from standard results in the literature \citep{bertsekas1996neuro}. Therefore, when the stepsize sequence $\{\alpha_k\}$ is appropriately chosen and $K$ is large, we expect $\theta_{t,K}$ to approximate the solution of Eq. (\ref{eq:linear}), i.e., $\theta_{t,K}\approx (\Phi^\top D\Phi)^{-1}\Phi^\top D\mathcal{H}(\Phi \hat{\theta}_t)$.
Now in view of line 7 of Algorithm \ref{alg:target_network_QLFA}, the target network $\hat{\theta}_{t+1}$ is synchronized to $\theta_{t,K}$. Therefore $Q$-learning with target network is in effect performing a stochastic variant of the fixed-point iteration (\ref{eq:fi2}).

Note that on an aside, $Q$-learning with target network can be viewed as an online version of fitted $Q$-iteration. To see this, recall that in the linear function approximation setting, fitted $Q$-iteration updates the corresponding parameter $\{\Tilde{\theta}_t\}$ iteratively according to
\begin{align}\label{eq:FQI1}
    \Tilde{\theta}_{t+1}&={\arg\min}_{\Tilde{\theta}\in\mathbb{R}^d}\frac{1}{|\mathcal{N}|}\sum_{(s,a,s')\in \mathcal{N}}\left(\phi(s,a)^\top \Tilde{\theta}-\mathcal{R}(s,a)-\gamma \max_{a'\in\mathcal{A}}\phi(s',a')^\top \Tilde{\theta}_t\right)^2,
\end{align}
where $\mathcal{N}=\{(s,a,s')\}$ is a batch dataset generated in an i.i.d. manner as follows: $s\sim \mu(\cdot)$, $a\sim \pi_b(\cdot|s)$, and $s'\sim P_a(s,\cdot)$. Observe that Eq. (\ref{eq:FQI1}) is an empirical version of 
\begin{align}\label{eq:FQI2}
    \Tilde{\theta}_{t+1}&={\arg\min}_{\Tilde{\theta}\in\mathbb{R}^d}\|\Phi\Tilde{\theta}-\mathcal{H}(\Phi \Tilde{\theta}_t)\|_D^2.
\end{align}
In light of Eq. (\ref{eq:FQI2}), the inner-loop of Algorithm \ref{alg:target_network_QLFA} can be viewed as a stochastic gradient descent algorithm for solving the optimization problem in Eq. (\ref{eq:FQI2}) with a single trajectory of Markovian samples.

Revisiting Baird's counter-example (where $d=|\mathcal{S}||\mathcal{A}|$), recall that the fixed-point iteration (\ref{eq:fi2}) reduces to $\theta_{k+1}=\Phi^{-1}\mathcal{H}(\Phi\theta_k)=\mathcal{H}_\Phi(\theta_k)$. Since the operator $\mathcal{H}_\Phi(\cdot)$ is a contraction mapping as shown in Section \ref{subsec:classical_Q}, the fixed-point iteration (\ref{eq:fi2}) provably converges. As a result, $Q$-learning with target network as a stochastic variant of the fixed-point iteration (\ref{eq:fi2}) also converges.

\begin{proposition}\label{prop:a.s.}
Consider Algorithm \ref{alg:target_network_QLFA}. Suppose that Assumption \ref{as:MC} is satisfied, the feature matrix $\Phi$ is a square matrix (i.e., $d=|\mathcal{S}||\mathcal{A}|$), $\alpha_k\equiv \alpha\leq \frac{\lambda_{\min}(1-\gamma)^2}{130}$, and $K\geq t_\alpha+1$. Then the sample complexity to achieve $\mathbb{E}[\|\Phi \hat{\theta}_T-Q^*\|_\infty]<\epsilon$ is $\Tilde{\mathcal{O}}(\epsilon^{-2})$.
\end{proposition}

To further verify the stability, we conduct numerical simulations for the MDP example constructed in \cite{baird1995residual}. As we see, while classical semi-gradient $Q$-learning with linear function approximation diverges in Figure \ref{fig:Baird_divergence} (which agrees with \cite{baird1995residual}), $Q$-learning with target network converges as shown in Figure \ref{fig:Baird_TN}.

\begin{figure}[h]
\centering
\begin{minipage}{.4\textwidth}
  \includegraphics[width=.95\linewidth]{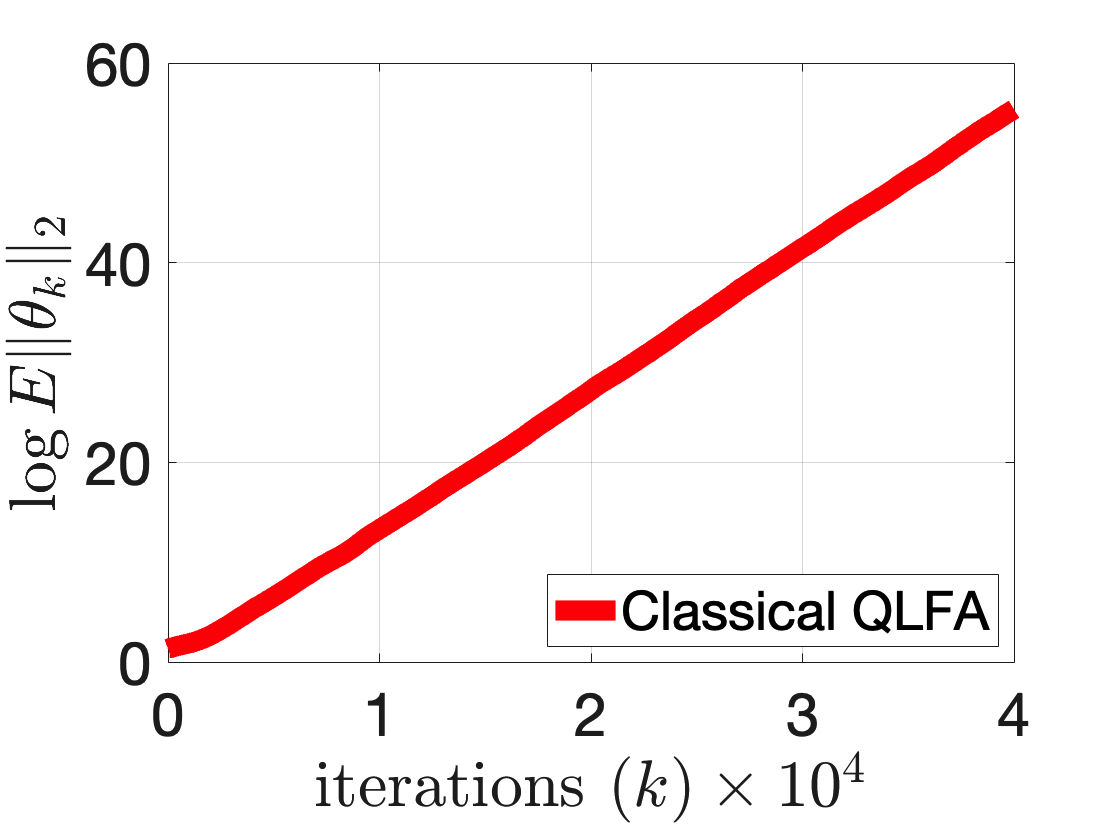}
  \caption{Classical Semi-Gradient $Q$-Learning}
  \label{fig:Baird_divergence}
\end{minipage}
\begin{minipage}{.4\textwidth}
  \includegraphics[width=.95\linewidth]{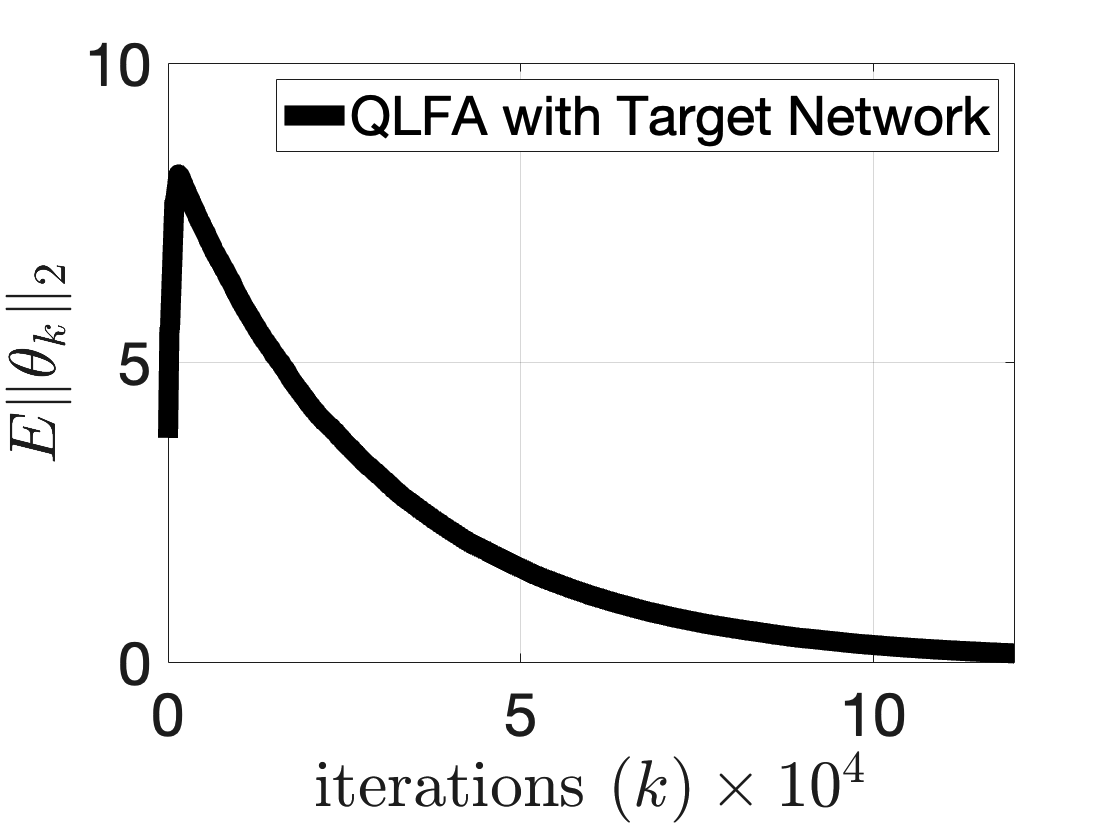}
  \caption{$Q$-Learning with Target Network}
  \label{fig:Baird_TN}
\end{minipage}
\end{figure}

\subsection{Insufficiency of Target Network}\label{subsec:insufficiency_of_TN}
The reason that $Q$-learning with target network overcomes the divergence for Baird's MDP example is essentially that \textit{the projected Bellman operator reduces to the regular Bellman operator (which is a contraction mapping) when we have a complete basis}. However, this is not the case in general. In the projected Bellman equation (\ref{eq:PBE}), the Bellman operator $\mathcal{H}(\cdot)$ is a contraction mapping with respect to the $\ell_\infty$-norm $\|\cdot\|_\infty$, and the projection operator $\text{Proj}_{\mathcal{W}}$ is a non-expansive mapping with respect to the projection norm, in this case the weighted $\ell_2$-norm $\|\cdot\|_D$. Due to the norm mismatch, the composed operator $\text{Proj}_{\mathcal{W}}\mathcal{H}(\cdot)$ in general is not a contraction mapping with respect to any norm. This is the \textit{fundamental reason} for the divergence of $Q$-learning with linear function approximation, and introducing target network alone does not overcome this issue, as evidenced by the following MDP example.

\begin{example}[MDP Example in \cite{chen2019finitesample}]\label{example1}
Consider an MDP with state-space $\mathcal{S}=\{s_1,s_2\}$ and action-space $\mathcal{A}=\{a_1,a_2\}$. Regardless of the present state, taking action $a_1$ results in state $s_1$ with probability $1$, and taking action $a_2$ results in state $s_2$ with probability $1$. The reward function is defined as $\mathcal{R}(s_1,a_1)=1$, $\mathcal{R}(s_1,a_2)=\mathcal{R}(s_2,a_1)=2$, and $\mathcal{R}(s_2,a_2)=4$. We construct the approximating linear sub-space with a single basis vector: $\Phi=[\phi(s_1,a_1),\phi(s_1,a_2),\phi(s_2,a_1),\phi(s_2,a_2)]^\top=[1,2,2,4]^\top$. The behavior policy is to take each action with equal probability. In this example, after straightforward calculation, we have the following result.

\begin{lemma}\label{le:PBE}
Eq. (\ref{eq:31}) is explicitly given by $\theta=1+\frac{9\gamma}{10}\theta+\frac{3\gamma\theta}{10} (\mathbb{I}_{\{\theta\geq 0\}}-\mathbb{I}_{\{\theta< 0\}})$.
\end{lemma}
When the discount factor $\gamma$ is in the interval $(5/6,1)$, for any positive initialization $\theta_0>0$, it is clear that performing fixed-point iteration to solve Eq. (\ref{eq:31}) in this example leads to divergence. Since $Q$-learning with target network is a stochastic variant of such fixed-point iteration, it also diverges. Numerical simulations demonstrate that performing either classical semi-gradient $Q$-learning (cf. Figure \ref{fig:chen_semi}) or $Q$-learning with target network (cf. Figure \ref{fig:chen_TN}) leads to divergence for the MDP in Example \ref{example1}.

\begin{figure}[H]
\centering
\begin{minipage}{.4\textwidth}
  \includegraphics[width=.95\linewidth]{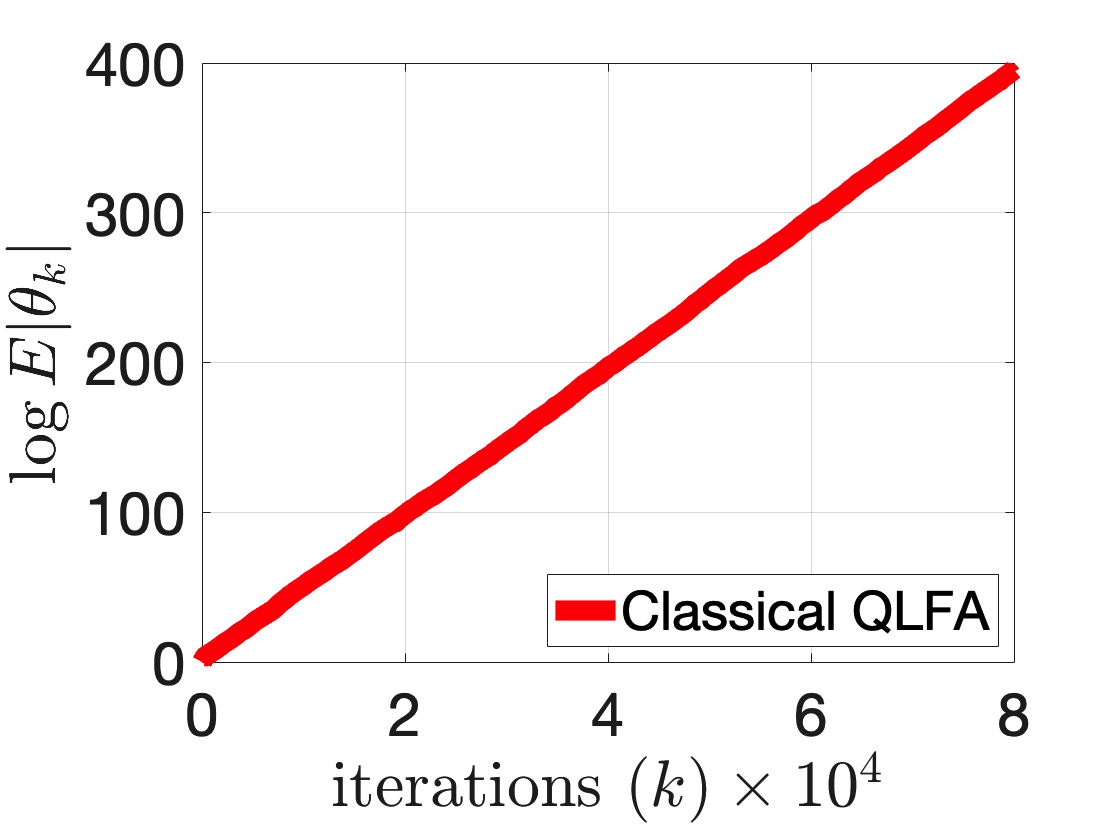}
  \caption{Classical Semi-Gradient $Q$-Learning}
  \label{fig:chen_semi}
\end{minipage}
\begin{minipage}{.4\textwidth}
  \includegraphics[width=.95\linewidth]{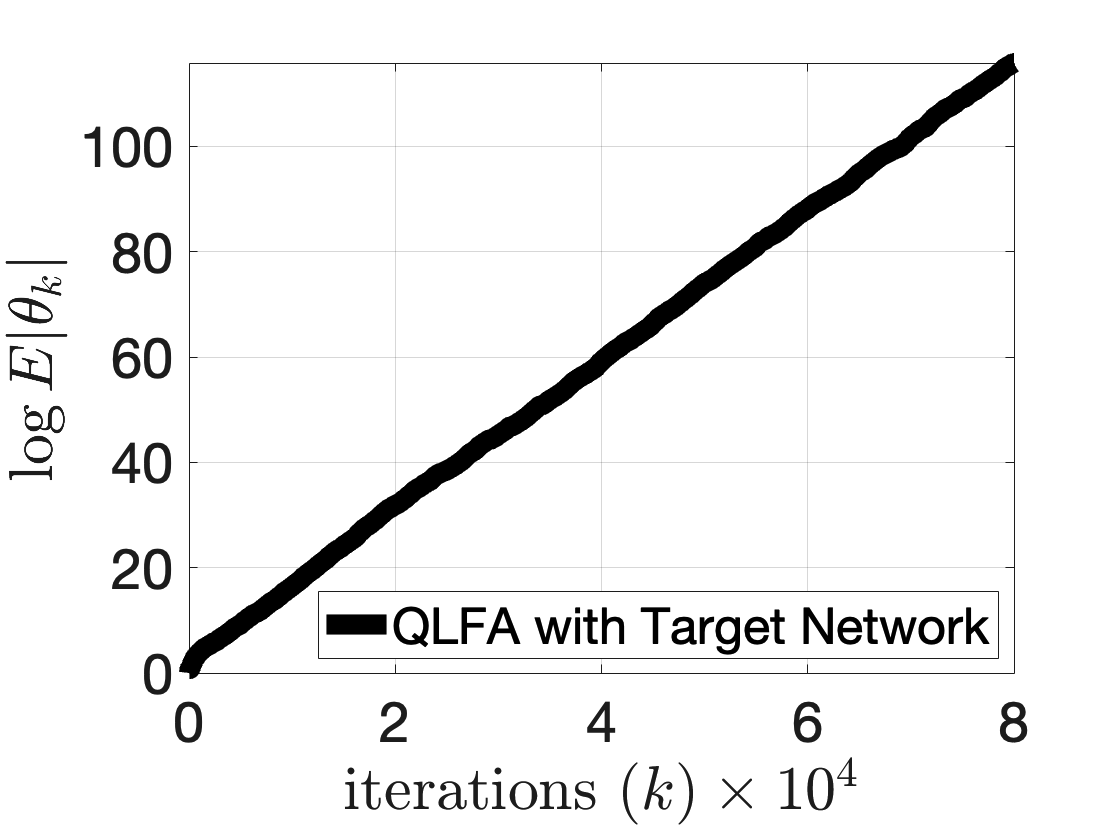}
  \caption{$Q$-Learning with Target Network}
  \label{fig:chen_TN}
\end{minipage}
\end{figure}
\end{example}

\subsection{Truncation to the Rescue}\label{subsec:TC}
The key ingredient we used to further overcome the divergence of $Q$-learning with target network is truncation. Recall from the previous  section that $Q$-learning with target network is trying to perform a stochastic variant of the fixed-point iteration (\ref{eq:fi2}), which can be equivalently written as
\begin{align}\label{eq:fi3}
    \Tilde{Q}_{t+1}=\text{Proj}_{\mathcal{W}}\mathcal{H}(\Tilde{Q}_t),
\end{align}
where we use $\Tilde{Q}_t$ to denote the $Q$-function estimate associated with the target network $\hat{\theta}_t$, i.e., $\Tilde{Q}_t=\Phi \hat{\theta}_t$. To motivate the truncation technique, we next analyze the update (\ref{eq:fi3}), whose behavior in terms of stability aligns with the behavior of $Q$-learning with target network, as explained in the previous section. First note that Eq. (\ref{eq:fi3}) is equivalent to
\begin{align*}
    \Tilde{Q}_{t+1}-Q^*=\mathcal{H}(\Tilde{Q}_t)-\mathcal{H}(Q^*)+\text{Proj}_{\mathcal{W}}\mathcal{H}(\Tilde{Q}_t)-\mathcal{H}(\Tilde{Q}_t).
\end{align*}
A simple calculation using triangle inequality, the contraction property of $\mathcal{H}(\cdot)$, and telescoping yields the following error bound of the iterative algorithm (\ref{eq:fi3}):
\begin{align*}
    \|\Tilde{Q}_{t+1}-Q^*\|_\infty
    \leq \gamma^{t+1}\|\Tilde{Q}_0-Q^*\|_\infty+\sum_{i=0}^t\gamma^{t-i}\underbrace{\|\text{Proj}_{\mathcal{W}}\mathcal{H}(\Tilde{Q}_i)-\mathcal{H}(\Tilde{Q}_i)\|_\infty}_{A_i}.
\end{align*}
The problem with the previous analysis is that the term $A_i$ (which captures the error due to using linear function approximation) is not necessarily bounded unless using a complete basis or knowing in prior that $\{\Tilde{Q}_t\}$ is always contained in a bounded set. The possibility that such function approximation error can be unbounded is an alternative explanation to the divergence of $Q$-learning with linear function approximation. This is true for arbitrary function approximation (including neural network) as well since it is in general not possible to uniformly approximate unbounded functions.

Suppose we are able to somehow control the size of the estimate $\Tilde{Q}_t$ so that it is always contained in a bounded set. Then the term $A_i$ is guaranteed to be finite, and well captures the approximation power of the chosen function class. To achieve the boundedness of the associated $Q$-function estimate $\Tilde{Q}_t$ of the target network, tracing back to Algorithm \ref{alg:target_network_QLFA}, a natural approach is to first project $\Phi\hat{\theta}_t$ onto the $\ell_\infty$-norm ball $B_r:=\{Q\in\mathbb{R}^{|\mathcal{S}||\mathcal{A}|}\mid \|Q\|_\infty\leq r\}$ before using it as the target $Q$-function in the inner-loop, resulting in Algorithm \ref{alg:middle_guy} presented in the following.

\begin{algorithm}[h]\caption{Impractical $Q$-Learning with Linear Function Approximation: Target Network and Projection}\label{alg:middle_guy}
\begin{algorithmic}[1] 
	\STATE {\bfseries Input:} Integers $T$, $K$, initializations $\theta_{t,0}=\bm{0}$ for all $t=0,1,...,T-1$ and $\hat{\theta}_0=\bm{0}$, behavior policy $\pi_b$
	\FOR{$t=0,1,\cdots,T-1$}
	\FOR{$k=0,1,\cdots,K-1$}
	\STATE Sample $A_k\sim \pi_b(\cdot|S_k)$, $S_{k+1}\sim P_{A_k}(S_k,\cdot)$
	\STATE $\theta_{t,k+1}=\theta_{t,k}+\alpha_k\phi(S_k,A_k)(\mathcal{R}(S_k,A_k)+\gamma\max_{a'\in\mathcal{A}}{\Tilde{Q}_t(S_{k+1},a')}-\phi(S_k,A_k)^\top \theta_{t,k})$
	\ENDFOR
	\STATE $\hat{\theta}_{t+1}=\theta_{t,K}$
	\STATE $\Tilde{Q}_{t+1}=\Pi_{B_r}\Phi \hat{\theta}_{t+1}$
	\STATE $S_0=S_K$
	\ENDFOR
	\STATE\textbf{Output:} $\hat{\theta}_T$
\end{algorithmic}
\end{algorithm}

In line 8 of Algorithm \ref{alg:middle_guy}, the operator $\Pi_{B_r}$ stands for the projection onto the $\ell_\infty$-norm ball $B_r$ with respect to some suitable norm $\|\cdot\|$. The specific norm $\|\cdot\|$ chosen to perform the projection turns out to be irrelevant as result of a key observation between truncation and projection.

Algorithm \ref{alg:middle_guy} although stabilizes the $Q$-function estimate $\Tilde{Q}_t$, it is not implementable in practice. To see this, recall that the whole point of using linear function approximation is to avoid working with $|\mathcal{S}||\mathcal{A}|$ dimensional objects. However, to implement Algorithm \ref{alg:middle_guy} line 8, one has to first compute $\Phi \hat{\theta}_{t+1}\in\mathbb{R}^{|\mathcal{S}||\mathcal{A}|}$, and then project it onto $B_r$. Therefore, the last difficulty we need to overcome is to find a way to implement Algorithm \ref{alg:middle_guy} without working with $|\mathcal{S}||\mathcal{A}|$ dimensional objects. The solution relies on the following observation.  

\begin{lemma}\label{le:truncation}
For any $x\in\mathbb{R}^{|\mathcal{S}||\mathcal{A}|} $ and any weighted $\ell_p$-norm $\|\cdot\|$ (the weights can be arbitrary and $p\in [1,\infty]$), we have $\ceil{x}\in \arg\min_{y\in B_r}\|x-y\|$.
\end{lemma}
\begin{remark}
Note that $\arg\min_{y\in B_r}\|x-y\|$ is in general a set because the projection may not be unique. As an example, observe that any point in the set $\{(x,1)\mid x\in [-1,1]\}$ is a projection of the point $(0,2)$ onto the $\ell_\infty$-norm unit  ball $\{(x,y)\mid x,y\in [-1,1]\}$ with respect to the $\ell_\infty$-norm.
\end{remark}
Lemma \ref{le:truncation} states that for any $x\in\mathbb{R}^{|\mathcal{S}||\mathcal{A}|}$, if we simply truncate $x$ at $r$, the resulting vector must belong to the projection set of $x$ onto the $\ell_\infty$-norm ball with radius $r$, for a wide class of projection norms. This seemingly simple but important result enables us to replace projection $\Pi_{B_r}(\cdot)$ by truncation $\ceil{\cdot}$ in line 8 of Algorithm \ref{alg:middle_guy}: 
\begin{align*}
    {\Tilde{Q}_{t+1}=\Pi_{B_r}\Phi \hat{\theta}_{t+1}}\quad \longrightarrow\quad\Tilde{Q}_{t+1}=\ceil{\Phi \hat{\theta}_{t+1}}.
\end{align*}
Unlike projection, truncation is a component-wise operation. Hence $\Tilde{Q}_{t+1}=\ceil{\Phi \hat{\theta}_{t+1}}$ is equivalent to $\Tilde{Q}_{t+1}(s,a)=\ceil{\phi(s,a)^\top \hat{\theta}_{t+1}}$ for all $(s,a)$.

The last issue is that we need to perform truncation for all state-action pairs $(s,a)$, which as illustrated earlier, violates the purpose of doing function approximation. However, observe that the target network is used in line 5 of Algorithm \ref{alg:middle_guy}, where only the components of $\Tilde{Q}_t$ visited by the sample trajectory is needed to perform the update. In light of this observation, instead of truncating $\phi(s,a)^\top \hat{\theta}_t$ for all $(s,a)$, we only need to truncate $\phi(S_{k+1},a')^\top \hat{\theta}_t$ in Algorithm \ref{alg:middle_guy} line 5, which leads to our stable version of $Q$-learning with linear function approximation in Algorithm \ref{alg:main}. The following proposition shows that target network and truncation together stabilized $Q$-learning with linear function approximation, and serves as a middle step to prove Theorem \ref{thm:main}.

\begin{proposition}\label{prop:outer_loop}
The following inequality holds:
\begin{align}
    \mathbb{E}[&\|\hat{Q}_T-Q^*\|_\infty]\leq \gamma^T\|\hat{Q}_0-Q^*\|_\infty+\frac{\mathcal{E}_{\text{approx}}}{1-\gamma}+\sum_{i=0}^{T-1}\gamma^{T-i-1}\mathbb{E}[\|\hat{Q}_{i+1}-\ceil{\text{Proj}_{\mathcal{W}}\mathcal{H}(\hat{Q}_i)}\|_\infty].\label{eq:81}
\end{align}
\end{proposition}

Because of truncation, the error due to using function approximation is bounded, and is captured by $\mathcal{E}_{\text{approx}}$. This is crucial to prevent the divergence of $Q$-learning with linear function approximation. The last term in Eq. (\ref{eq:81}) captures the error in the inner-loop of Algorithm \ref{alg:main}, and eventually contribute to the terms $E_2$ and $E_3$ in Eq. (\ref{eq:finite-sample-bounds}).

Revisiting Example \ref{example1}, where either semi-gradient $Q$-learning or $Q$-learning with target network diverges, Algorithm \ref{alg:main} converges as demonstrated in Figure \ref{fig:chen_TNTC}. Moreover, observe that Algorithm \ref{alg:main} seems to converge to a positive scalar, which we denote by $\theta^*$. As a result, the policy $\pi$ induced greedily from $\Phi \theta^*$ is to always take action $a_2$. It can be easily verified that $\pi$ is indeed the optimal policy. This is an interesting observation since the optimal $Q$-function $Q^*$ in this case does not belong to the linear sub-space $\mathcal{W}$ (which is spanned by a single basis vector $(1,2,2,4)^\top$). Nevertheless performing Algorithm \ref{alg:main} converges and the induced policy is optimal. Figure \ref{fig:baird_TNTC} shows that Algorithm \ref{alg:main} also converges for the Baird's MDP example.

\begin{figure}[H]
\centering
\begin{minipage}{.4\textwidth}
  \includegraphics[width=.95\linewidth]{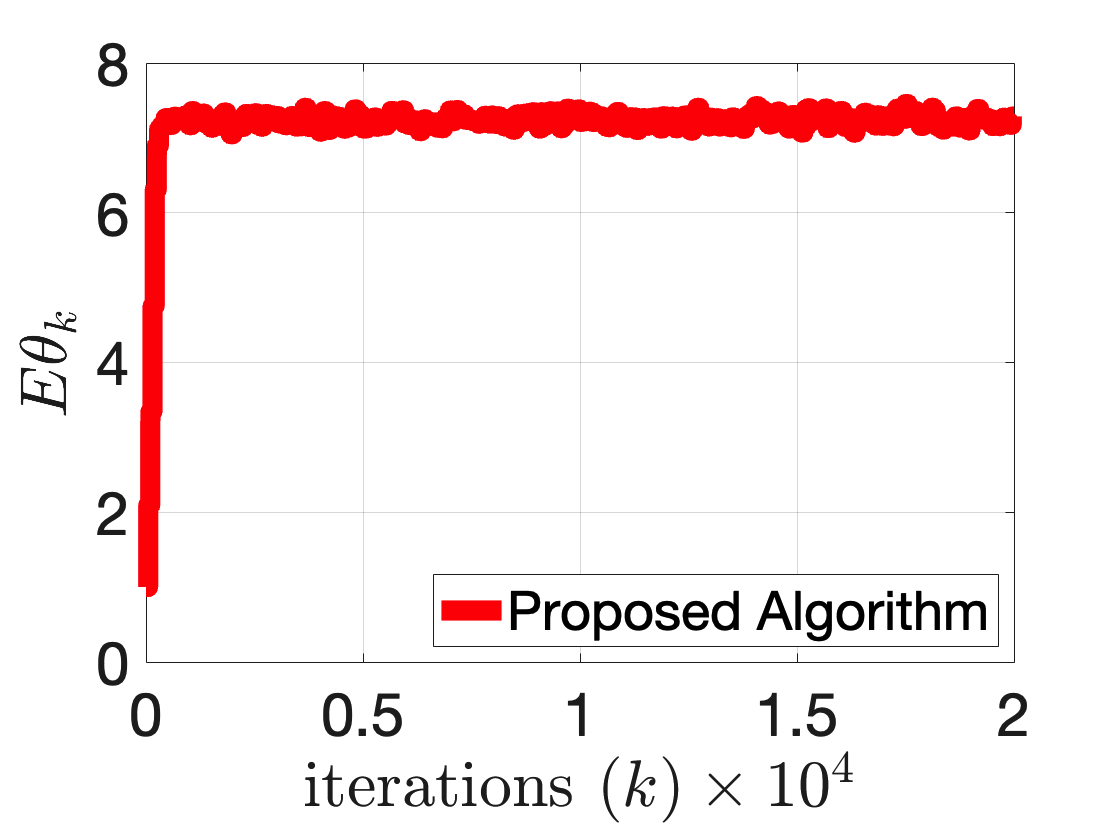}
  \caption{Algorithm \ref{alg:main} for Baird's MDP Example}
  \label{fig:chen_TNTC}
\end{minipage}
\begin{minipage}{.4\textwidth}
  \includegraphics[width=.95\linewidth]{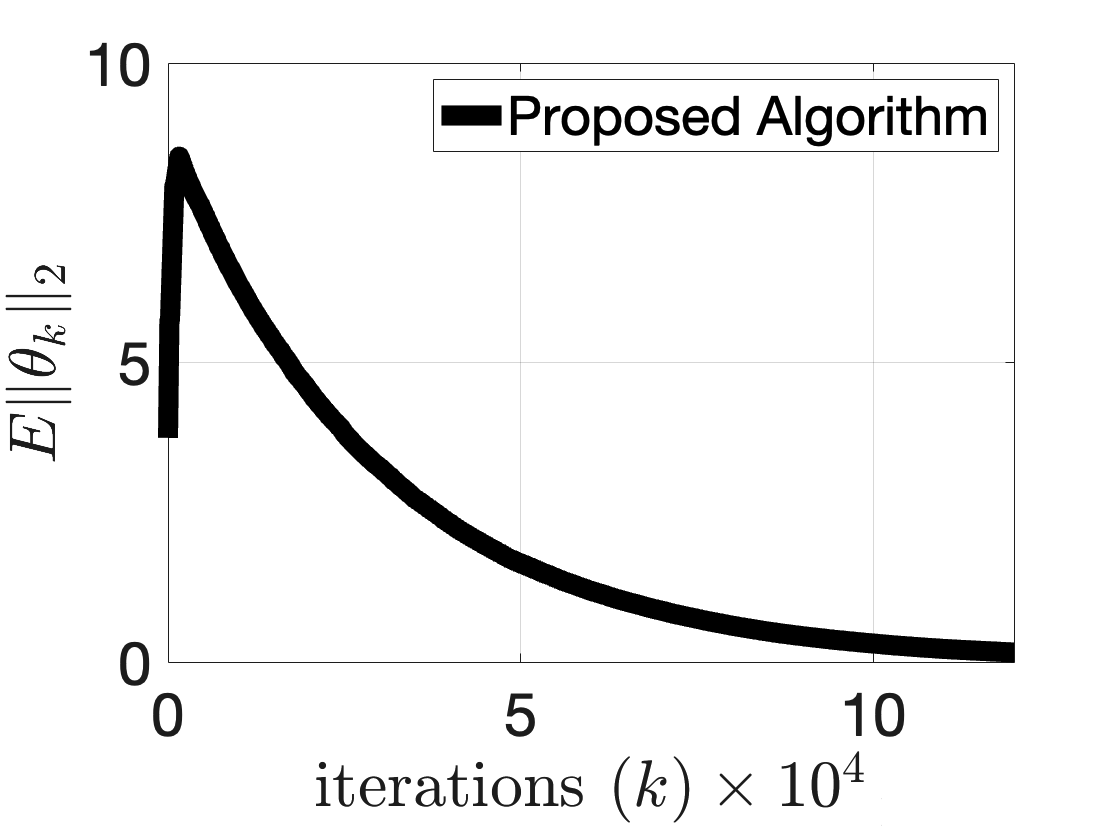}
  \caption{Algorithm \ref{alg:main} for Example \ref{example1}}
  \label{fig:baird_TNTC}
\end{minipage}
\end{figure}

\section{Conclusion and Future Work}\label{sec:conclusion}
This work makes contributions towards the understanding of $Q$-learning with function approximation. In particular, we show that by adding target network and truncation, the resulting $Q$-learning with linear function approximation is provably stable, and achieves the optimal $\Tilde{\mathcal{O}}(\epsilon^{-2})$ sample complexity up to a function approximation error. Furthermore, the establishment of our results do not require strong assumptions (e.g. linear MDP, strong negative drift assumption, sufficiently small discount factor $\gamma$, etc.) as in related literature. There are two immediate future directions in this line of work. One is to improve the function approximation error, and the second is to extend the results of this work to using neural network approximation, i.e., the Deep $Q$-Network. The detailed plan is provided in Appendix \ref{ap:future}.

\section*{Acknowledgement} We thank Prof. Csaba Szepesvari from University of Alberta for the insightful comments and suggestions about this work.

\bibliographystyle{imsart-nameyear}
\bibliography{references}

\begin{appendix}
\section{Divergent MDP Example in \cite{baird1995residual}}\label{ap:Baird}
The MDP instance constructed in \cite{baird1995residual} is presented in Figure \ref{fig:counterexample}. As we see, the state-space is $\mathcal{S}=\{1,2,...,7\}$ and action-space is $\mathcal{A}=\{\text{solid}, \text{dash}\}$. Regardless of the present state, the dash action takes the agent to one of the states $1,2,...,6$, each with equal probability, while the solid action takes the agent to state $7$ with probability $1$. The reward is identically equal to zero for all transitions, and the behavior policy $\pi_b$ is to take each action (solid or dash) with equal probability. 

\begin{figure}[h!]
	\centering
	\includegraphics[width=0.6\textwidth]{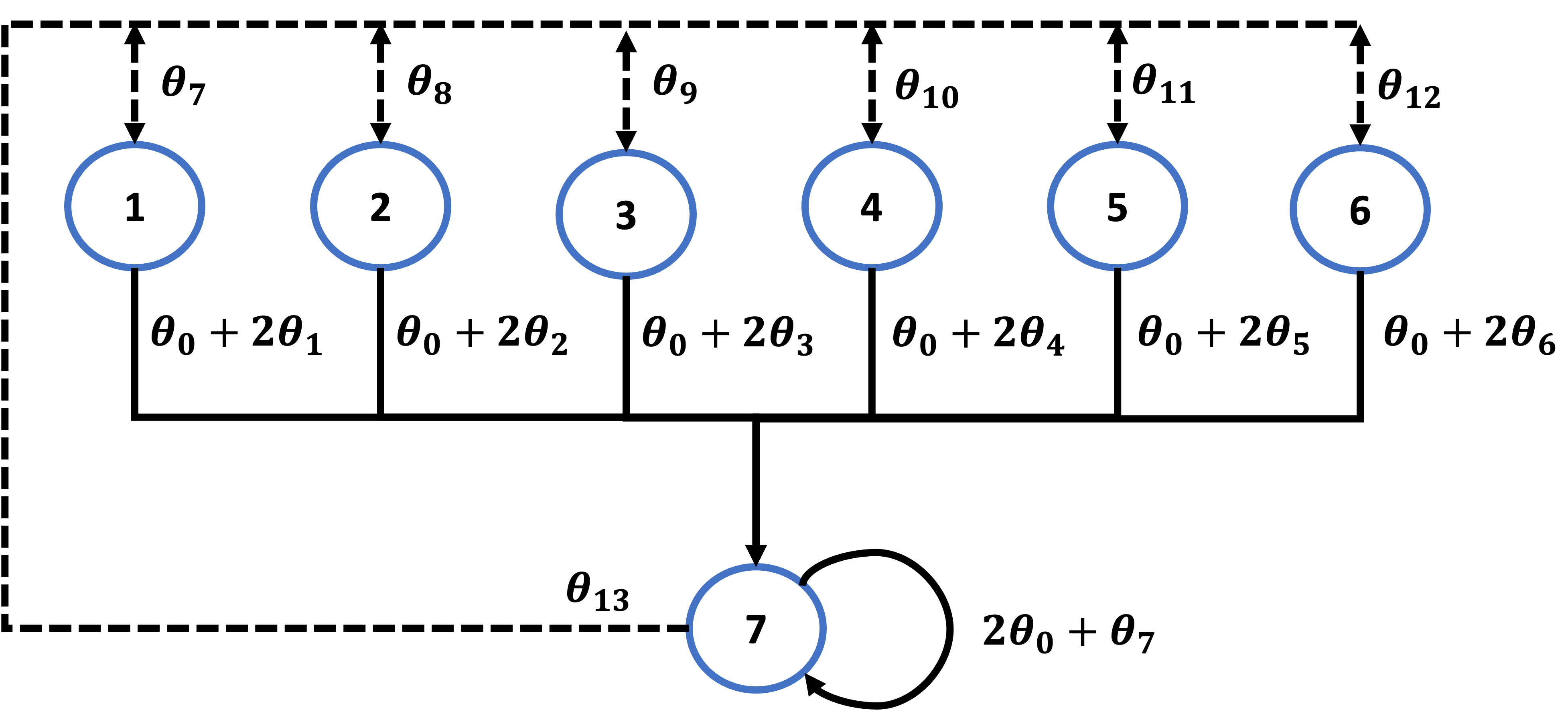}
	\caption{Baird’s counter-example \citep{baird1995residual}}\label{fig:counterexample}
\end{figure}

The $14$ basis vectors used for linear approximation are also presented in Figure \ref{fig:counterexample}. For example, the $Q$-function at state $1$ taking solid action is approximated by $\theta_0+2\theta_1$. One can easily check that the basis vectors are linearly independent. Hence this is essentially a change of basis. Performing classical semi-gradient $Q$-learning with linear function approximation leads to divergence in this example, as demonstrated in \cite{baird1995residual}, and also in Figure \ref{fig:Baird_divergence} of this work.

\section{Proof of Theorem \ref{thm:main}}

\subsection{Analysis of the Outer-Loop (Proof of Proposition \ref{prop:outer_loop})}\label{ap:outer_loop}
Recall that we denote $\hat{Q}_t=\ceil{\Phi\hat{\theta}_t}$. Using the fact that $Q^*=\mathcal{H}(Q^*)$, we have for any $t\geq 0$ that
\begin{align*}
    \hat{Q}_t-Q^*&=\hat{Q}_t-\mathcal{H}(Q^*)\\
    &=\mathcal{H}(\hat{Q}_{t-1})-\mathcal{H}(Q^*)+\hat{Q}_t-\ceil{\text{Proj}_{\mathcal{W}}\mathcal{H}(\hat{Q}_{t-1})}+\ceil{\text{Proj}_{\mathcal{W}}\mathcal{H}(\hat{Q}_{t-1})}-\mathcal{H}(\hat{Q}_{t-1}).
\end{align*}
It follows that
\begin{align*}
    \|\hat{Q}_t-Q^*\|_\infty
    \leq \;&\|\mathcal{H}(\hat{Q}_{t-1})-\mathcal{H}(Q^*)\|_\infty+\|\hat{Q}_t-\ceil{\text{Proj}_{\mathcal{W}}\mathcal{H}(\hat{Q}_{t-1})}\|_\infty\\
    &+\|\ceil{\text{Proj}_{\mathcal{W}}\mathcal{H}(\hat{Q}_{t-1})}-\mathcal{H}(\hat{Q}_{t-1})\|_\infty\\
    \leq \;&\gamma\|\hat{Q}_{t-1}-Q^*\|_\infty+\|\hat{Q}_t-\ceil{\text{Proj}_{\mathcal{W}}\mathcal{H}(\hat{Q}_{t-1})}\|_\infty+\mathcal{E}_{\text{approx}},
\end{align*}
where the last line follows from $\mathcal{H}(\cdot)$ being a $\gamma$-contraction mapping with respect to $\|\cdot\|_\infty$, and the definition of $\mathcal{E}_{\text{approx}}$.

Repeatedly using the previous inequality and then taking expectation on both sides of the resulting inequality, and we have for any $T\geq 0$:
\begin{align}\label{eq:20}
    \mathbb{E}[\|\hat{Q}_T-Q^*\|_\infty]\leq \gamma^T\|\hat{Q}_0-Q^*\|_\infty+\sum_{i=0}^{T-1}\gamma^{T-i-1}\mathbb{E}[\|\hat{Q}_{i+1}-\ceil{\text{Proj}_{\mathcal{W}}\mathcal{H}(\hat{Q}_i)}\|_\infty]+\frac{\mathcal{E}_{\text{approx}}}{1-\gamma}.
\end{align}
This proves Proposition \ref{prop:outer_loop}.
The remaining task is to control $\mathbb{E}[\|\hat{Q}_{i+1}-\ceil{\text{Proj}_{\mathcal{W}}\mathcal{H}(\hat{Q}_i)}\|_\infty]$ for any $i=0,...,T-1$. First of all, since $\hat{Q}_t=\ceil{\Phi \hat{\theta}_t}=\ceil{\Phi \theta_{t-1,K}}$ and $\|\ceil{Q_1}-\ceil{Q_2}\|_\infty\leq \|Q_1-Q_2\|_\infty$ for any $Q_1,Q_2\in\mathbb{R}^{|\mathcal{S}||\mathcal{A}|}$, we have
\begin{align*}
    \mathbb{E}[\|\hat{Q}_{i+1}-\ceil{\text{Proj}_{\mathcal{W}}\mathcal{H}(\hat{Q}_i)}\|_\infty]\leq \mathbb{E}[\|\Phi \theta_{i,K}-\text{Proj}_{\mathcal{W}}\mathcal{H}(\hat{Q}_i)\|_\infty].
\end{align*}
To further bound the RHS of the previous inequality, we need to analyze the inner-loop of Algorithm \ref{alg:main}, which is done in the next section.

\subsection{Analysis of the Inner-Loop}
We begin by presenting the inner-loop of Algorithm \ref{alg:main}.
\begin{algorithm}[H]\caption{Inner-Loop of Algorithm \ref{alg:main}}\label{alg:inner-loop}
\begin{algorithmic}[1] 
	\STATE {\bfseries Input:} Integer $K$, initialization $\theta_0=\bm{0}$, target network $\hat{\theta}$, behavior policy $\pi_b$
	\FOR{$k=0,1,\cdots,K-1$}
	\STATE Sample $A_k\sim \pi_b(\cdot|S_k)$, $S_{k+1}\sim P_{A_k}(S_k,\cdot)$
	\STATE $\theta_{k+1}=\theta_k+\alpha_k\phi(S_k,A_k)(\mathcal{R}(S_k,A_k)+\gamma\max_{a'\in\mathcal{A}}\ceil{\phi(S_{k+1},a')^\top\hat{\theta}}-\phi(S_k,A_k)^\top \theta_k)$
	\ENDFOR
	\STATE\textbf{Output:} $\theta_K$
\end{algorithmic}
\end{algorithm}

In view of the main update equation, Algorithm \ref{alg:inner-loop} is a Markovian linear stochastic approximation algorithm for solving the following linear system of equations:
\begin{align*}
    -\Phi^\top D\Phi\theta+\Phi^\top D\mathcal{H}(\Phi\hat{\theta})=0.
\end{align*}
Since the matrix $-\Phi^\top D\Phi$ is negative definite, the finite-sample guarantees follow from standard results in the literature \citep{srikant2019finite,chen2019finitesample}. Specifically, we will apply \cite{chen2019finitesample} Corollary 2.1 to establish the result. To make this paper self-contained, we first present Corollary 2.1 of \cite{chen2019finitesample} in the following.

\begin{theorem}[Corollary 2.1 of \cite{chen2019finitesample}]\label{thm:chen2019}
Consider the Markovian stochastic approximation algorithm with an arbitrary initialization $x_0\in\mathbb{R}^d$:
\begin{align}\label{sa:chen2019}
    x_{k+1}=x_k+\alpha F(x_k,Y_k).
\end{align}
Suppose that
\begin{enumerate}[(1)]
    \item The finite-state  Markov chain $\{Y_k\}$ has a unique stationary distribution $\nu$, and it holds for any $k\geq  0$ that $\max_{y\in\mathcal{Y}}\|P^k(y,\cdot)-\nu(\cdot)\|_{\text{TV}}\leq C\rho^k$ for some constant $C>0$ and $\rho\in (0,1)$.
    \item There exist constants $L_1,L_2>0$ such that the operator $F(\cdot,\cdot)$ satisfies $\|F(x_1,y)-F(x_2,y)\|_2\leq L_1\|x_1-x_2\|_2$ and $\|F(\bm{0},y)\|_2\leq L_2$ for any $x_1,x_2\in\mathbb{R}^d$ and $y\in\mathcal{Y}$.
    \item The equation $\bar{F}(x)=\mathbb{E}_{Y\sim \nu}[F(x,Y)]=0$ has a unique solution $x^*\in\mathbb{R}^d$, and the following inequality holds for all $x\in\mathbb{R}^d$: $(x-x^*)^\top \bar{F}(x)\leq -\kappa \|x-x^*\|_2^2$, where $\kappa>0$ is a positive constant.
    \item The stepsize sequence $\{\alpha_k\}$ is a constant sequence (i.e., $\alpha_k\equiv \alpha$), and $\alpha$ is chosen such that $\alpha \tau_\alpha\leq \frac{\kappa}{130\max(L_1,L_2)^2}$, where $\tau_\alpha:=\min\{k\geq 0\;:\;\max_{y\in\mathcal{Y}}\|P^k(y,\cdot)-\nu(\cdot)\|_{\text{TV}}\leq\alpha\}$.
\end{enumerate}
Then we have for any $k\geq \tau_\alpha$ that
\begin{align*}
    \mathbb{E}[\|x_k-x^*\|_2^2]\leq (\|x_0\|_2+\|x_0-x^*\|_2+1)^2(1-\kappa \alpha)^{k-\tau_\alpha}+130(L_1\|x^*\|_2+L_2)^2\frac{\alpha \tau_\alpha}{\kappa}.
\end{align*}
\end{theorem}

To apply Theorem \ref{thm:chen2019}, we first rewrite the update equation in line 4 of Algorithm \ref{alg:inner-loop} in the form of stochastic approximation algorithm (\ref{sa:chen2019}). Then we verify that Conditions (1) -- (4) are satisfied.

\begin{itemize}
\item \textit{Reformulation.} For any $k\geq 0$, let $Y_k=(S_k,A_k,S_{k+1})$, which is clearly a Markov chain with state-space given by $\mathcal{Y}=\{y=(s,a,s')\mid s\in\mathcal{S},\pi_b(a|s)>0,P_a(s,s')>0\}$. Define the function $F:\mathbb{R}^d\times\mathcal{Y}\mapsto\mathbb{R}^d$ by 
\begin{align*}
    F(\theta,s,a,s')=\phi(s,a)\left(\mathcal{R}(s,a)+\gamma\max_{a'\in\mathcal{A}}\ceil{\phi(s',a')^\top\hat{\theta}}-\phi(s,a)^\top \theta\right)
\end{align*}
for any $\theta\in\mathbb{R}^d$ and $y=(s,a,s')\in\mathcal{Y}$. Then the update equation of Algorithm \ref{alg:inner-loop} can be equivalently written as
\begin{align}
    \theta_{k+1}=\theta_k+\alpha F(\theta_k,Y_k).
\end{align}
\item \textit{Verification of Condition (1).} Under Assumption \ref{as:MC}, the Markov chain $\{Y_k\}$ has a unique stationary distribution $\nu$, which is given by $\nu(s,a,s')=\mu(s)\pi(a|s)P_a(s,s')$ for all $(s,a,s')\in\mathcal{Y}$. In addition, we have for any $y=(s,a,s')\in\mathcal{Y}$ that
\begin{align*}
    \|P_{\pi_b}^k(y,\cdot)-\nu(\cdot)\|_{\text{TV}}&= \frac{1}{2}\sum_{(s_0,a_0,s_1)\in\mathcal{Y}}\left|P_{\pi_b}^{k-1}(s',s_0)-\mu(s_0)\right|\pi(a_0|s_0)P_{a_0}(s_0,s_1)\\
    &\leq \frac{1}{2}\sum_{s_0\in\mathcal{S}}\left|P_{\pi_b}^{k-1}(s',s_0)-\mu(s_0)\right|\\
    &\leq \max_{s\in\mathcal{S}}\|P_{\pi_b}^{k-1}(s,\cdot)-\mu(\cdot)\|_{\text{TV}}\\
    &\leq C\rho^{k-1}.
\end{align*}
\item \textit{Verification of Condition (2).} For any $x_1,x_2\in\mathbb{R}^d$ and $y=(s,a,s')\in\mathcal{Y}$, we have
\begin{align*}
    \|F(\theta_1,y)-F(\theta_2,y)\|_2&=\|\phi(s,a)\phi(s,a)^\top (\theta_1-\theta_2)\|_2\\
    &\leq \|\phi(s,a)\|_2^2\|\theta_1-\theta_2\|_2\\
    &\leq \|\theta_1-\theta_2\|_2.\tag{$\|\phi(s,a)\|_2\leq \|\phi(s,a)\|_1\leq 1$ for all $(s,a)$}
\end{align*}
Similarly, we have for any $y=(s,a,s')\in\mathcal{Y}$ that
\begin{align*}
    \|F(\bm{0},y)\|_2&=\left\|\phi(s,a)\left(\mathcal{R}(s,a)+\gamma\max_{a'\in\mathcal{A}}\ceil{\phi(s',a')^\top\hat{\theta}}\right)\right\|_2\\
    &\leq \left(1+\frac{\gamma}{1-\gamma}\right)\|\phi(s,a)\|_2\\
    &\leq \frac{1}{1-\gamma}.
\end{align*}
\item \textit{Verification of Condition (3).} By definition of $F(\cdot,\cdot)$, we have 
\begin{align*}
    \bar{F}(\theta)&=\mathbb{E}_{(S_k,A_k,S_{k+1})\sim \nu}\left[\phi(S_k,A_k)\left(\mathcal{R}(S_k,A_k)+\gamma\max_{a'\in\mathcal{A}}\ceil{\phi(S_{k+1},a')^\top\hat{\theta}}-\phi(S_k,A_k)^\top \theta\right)\right]\\
    &=-\Phi^\top D\Phi\theta+\Phi^\top D\mathcal{H}(\ceil{\Phi\hat{\theta}}),
\end{align*}
where we recall that $D\in\mathbb{R}^{|\mathcal{S}||\mathcal{A}|\times |\mathcal{S}||\mathcal{A}|}$ is a diagonal matrix with diagonal entries $\{\mu(s)\pi_b(a|s)\}_{(s,a)\in\mathcal{S}\times\mathcal{A}}$. Since $\Phi$ has linearly independent columns, the matrix $\Phi^\top D\Phi$ is invertible. Solving $\bar{F}(\theta)=0$ and we obtain $\theta^*=(\Phi^\top D\Phi)^{-1}\Phi^\top D\mathcal{H}(\ceil{\Phi\hat{\theta}})$. Furthermore, note that the matrix $\Phi^\top D\Phi$ is positive definite, whose smallest eigenvalue is denoted by $\lambda_{\min}$. Therefore we have for any $\theta\in\mathbb{R}^d$:
\begin{align*}
    (\theta-\theta^*)^\top\bar{F}(\theta)&=(\theta-\theta^*)^\top (\bar{F}(\theta)-\bar{F}(\theta^*))\\
    &=-(\theta-\theta^*)^\top \Phi^\top D\Phi(\theta-\theta^*)\\
    &\leq -\lambda_{\min}\|\theta-\theta^*\|_2^2.
\end{align*}
\item \textit{Verification of Condition (4).} This is satisfied due to our choice of the constant stepsize $\alpha$ in Theorem \ref{thm:main}. 
\end{itemize}

Now that all Conditions are satisfied. Apply Theorem \ref{thm:chen2019} and we obtain for any $k\geq t_\alpha+1$:
\begin{align}\label{eq:144}
    \mathbb{E}[\|\theta_k-\theta^*\|_2^2]\leq (\|\theta^*\|_2+1)^2(1-\lambda_{\min} \alpha)^{k-t_\alpha-1}+\frac{130}{(1-\gamma)^2}((1-\gamma)\|\theta^*\|_2+1)^2\frac{\alpha (t_\alpha+1)}{\lambda_{\min}},
\end{align}
where we used $\theta_0=\bm{0}$ in Algorithm \ref{alg:inner-loop}.
The last step is to provide an upper bound on $\|\theta^*\|_2$. Note that
\begin{align*}
    \|\theta^*\|_2&=\frac{1}{\lambda_{\min}^{1/2}}\lambda_{\min}^{1/2}\|\theta^*\|_2\\
    &\leq \frac{1}{\lambda_{\min}^{1/2}}\|\Phi\theta^*\|_D\\
    &=\frac{1}{\lambda_{\min}^{1/2}}\|\Phi(\Phi^\top D\Phi)^{-1}\Phi^\top D\mathcal{H}(\ceil{\Phi\hat{\theta}})\|_D\tag{$\bar{F}(\theta^*)=0$}\\
    &=\frac{1}{\lambda_{\min}^{1/2}}\|\text{Proj}_{\mathcal{W}}\mathcal{H}(\ceil{\Phi\hat{\theta}})\|_D\\
    &\leq \frac{1}{\lambda_{\min}^{1/2}}\|\mathcal{H}(\ceil{\Phi\hat{\theta}})\|_D\tag{$\text{Proj}_{\mathcal{W}}(\cdot)$ is non-expansive with respect to $\|\cdot\|_D$}\\
    &\leq \frac{1}{\lambda_{\min}^{1/2}(1-\gamma)}\|\bm{1}\|_D\tag{$-\frac{1}{1-\gamma}\bm{1}\leq \mathcal{H}(\ceil{Q})\leq \frac{1}{1-\gamma}\bm{1}$ for any $Q\in\mathbb{R}^{|\mathcal{S}||\mathcal{A}|}$}\\
    &=\frac{1}{\lambda_{\min}^{1/2}(1-\gamma)}.
\end{align*}
Substituting the previous upper bound we obtained for $\|\theta^*\|_2$ into Eq. (\ref{eq:144}) and we finally have for all $k\geq t_\alpha+1$:
\begin{align}\label{eq:21}
    \mathbb{E}[\|\theta_k-\theta^*\|_2^2]\leq \frac{4}{\lambda_{\min}(1-\gamma)^2}(1-\lambda_{\min} \alpha)^{k-t_\alpha-1}+\frac{520}{\lambda_{\min}^2(1-\gamma)^2}\alpha (t_\alpha+1).
\end{align}

\subsection{Putting Together}
In this section, we combine the analysis of the outer-loop and the inner-loop to establish the overall finite-sample bounds of Algorithm \ref{alg:main}. Denote $\theta^*_t=(\Phi^\top D\Phi)^{-1}\Phi^\top D\mathcal{H}(\hat{Q}_t)$. Note that we have $\Phi\theta^*_t=\text{Proj}_{\mathcal{W}}\mathcal{H}(\hat{Q}_t)$. Using the fact that $\|\cdot\|_\infty\leq \|\cdot\|_2$ and we obtain for any $0\leq i\leq T$:
\begin{align*}
    \mathbb{E}[\|\Phi \theta_{i,K}-\text{Proj}_{\mathcal{W}}\mathcal{H}(\hat{Q}_i)\|_\infty]&=\mathbb{E}[\|\Phi (\theta_{i,K}-\theta^*_i)\|_\infty]\\
    &\leq \mathbb{E}[\|\Phi\|_\infty\| \theta_{i,K}-\theta^*_i\|_\infty]\\
    &\leq \mathbb{E}[\|\theta_{i,K}-\theta^*_i\|_\infty]\tag{$\|\phi(s,a)\|_1\leq 1$ for all $(s,a)$}\\
    &\leq \mathbb{E}[\|\theta_{i,K}-\theta^*_i\|_2]\\
    &\leq \left(\mathbb{E}[\|\theta_{i,K}-\theta^*_i\|_2^2]\right)^{1/2}\tag{Jensen's inequality}\\
    &\leq \left(\frac{4}{\lambda_{\min}(1-\gamma)^2}(1-\lambda_{\min} \alpha)^{K-t_\alpha-1}+\frac{520}{\lambda_{\min}^2(1-\gamma)^2}\alpha (t_\alpha+1)\right)^{1/2}\tag{Eq. (\ref{eq:21})}\\
    &\leq \frac{2}{\lambda_{\min}^{1/2}(1-\gamma)}(1-\lambda_{\min}\alpha)^{\frac{K-t_\alpha-1}{2}}+\frac{24}{\lambda_{\min}(1-\gamma)}\sqrt{\alpha (t_\alpha+1)},
\end{align*}
where the last line follows from $\sqrt{a+b}\leq \sqrt{a}+\sqrt{b}$ for any $a,b\geq 0$.

Substituting the previous inequality into Eq. (\ref{eq:20}), and we obtain the overall finite-sample guarantees of Algorithm \ref{alg:main}:
\begin{align*}
    \mathbb{E}[\|\hat{Q}_T-Q^*\|_\infty]\leq\;& \gamma^T\|\hat{Q}_0-Q^*\|_\infty+\frac{2}{\lambda_{\min}^{1/2}(1-\gamma)^2}(1-\lambda_{\min}\alpha)^{\frac{K-t_\alpha-1}{2}}\\
    &+\frac{24}{\lambda_{\min}(1-\gamma)^2}\sqrt{\alpha (t_\alpha+1)}+\frac{\mathcal{E}_{\text{approx}}}{1-\gamma}.
\end{align*}

In view of the finite-sample guarantee, to obtain $\mathbb{E}[\|\hat{Q}_T-Q^*\|_\infty]\leq \epsilon+\frac{\mathcal{E}_{\text{approx}}}{1-\gamma}$ for a given accuracy $\epsilon$, the number of sample required is of the size
\begin{align*}
    \mathcal{O}\left(\epsilon^{-2}\log^2(1/\epsilon)\right)\Tilde{\mathcal{O}}\left(\frac{1}{(1-\gamma)^4}\right).
\end{align*}

\section{Proof of All Technical Results in Section \ref{sec:roadmap}}

\subsection{Proof of Proposition \ref{prop:a.s.}}\label{ap:a.s.convergence}
The proof is identical to that of Theorem \ref{thm:main}, and hence is omitted.

\subsection{Proof of Proposition \ref{prop:outer_loop}}

See Appendix \ref{ap:outer_loop}.

\subsection{Proof of Lemma \ref{le:PBE}}
We first compute the transition probability matrix of the Markov chain $\{S_k\}$ under $\pi_b$. Since
\begin{align*}
    P_{a_1}=\begin{bmatrix}
    1\;\;& 0\\
    1\;\;& 0
    \end{bmatrix}\quad \text{and}\quad P_{a_2}=\begin{bmatrix}
    0\;\;& 1\\
    0\;\;& 1
    \end{bmatrix},
\end{align*}
and $\pi(a|s)=1/2$ for any $a\in \{a_1,a_2\}$ and $s\in \{s_1,s_2\}$, we have $P_{\pi_b}=\frac{1}{2}I_2$. As a result, the unique stationary distribution $\mu$ of the Markov chain $\{S_k\}$ under $\pi_b$ is given by $\mu=(1/2,1/2)$. Therefore, the matrix $D\in\mathbb{R}^{|\mathcal{S}||\mathcal{A}|\times |\mathcal{S}||\mathcal{A}|}$ (defined before Theorem \ref{thm:main}) is given by $D=\frac{1}{4}I_4$. We next compute Eq. (\ref{eq:31}) in this example. First of all, by definition of the Bellman operator we have for any $\theta\in\mathbb{R}$ that
\begin{align*}
    [\mathcal{H}(\Phi \theta)](s_1,a_1)&=\mathcal{R}(s_1,a_1)+\gamma\mathbb{E}[\max_{a'\in\mathcal{A}}\phi(S_{k+1},a')\theta\mid S_k=s_1,A_k=a_1]\\
    &=\mathcal{R}(s_1,a_1)+\gamma \max_{a'\in\mathcal{A}}\phi(s_1,a')\theta\\
    &=\begin{dcases}
    1+2\gamma\theta,&\theta\geq 0,\\
    1+\gamma\theta,&\theta<0.
    \end{dcases}
\end{align*}
Similarly, we also have
\begin{align*}
    [\mathcal{H}(\Phi \theta)](s_1,a_2)&=\begin{dcases}
    2+4\gamma\theta,&\theta\geq 0,\\
    2+2\gamma\theta,&\theta<0.
    \end{dcases}&\quad [\mathcal{H}(\Phi \theta)](s_2,a_1)&=\begin{dcases}
    2+2\gamma\theta,&\theta\geq 0,\\
    2+\gamma\theta,&\theta<0.
    \end{dcases}\\
    [\mathcal{H}(\Phi \theta)](s_2,a_2)&=\begin{dcases}
    4+4\gamma\theta,&\theta\geq 0,\\
    4+2\gamma\theta,&\theta<0.
    \end{dcases}
\end{align*}
Therefore, Eq. (\ref{eq:31}) in the case of Example \ref{example1} is explicitly given by
\begin{align*}
    \theta&=(\Phi^\top D\Phi)^{-1}\Phi^\top D\mathcal{H}(\Phi\theta)\\
    &=\begin{dcases}
    \frac{1}{25}\begin{bmatrix}
    1\;& 2\;& 2\;&4
    \end{bmatrix}
    \begin{bmatrix}
    1+2\gamma\theta\\
    2+4\gamma\theta\\
    2+2\gamma\theta\\
    4+4\gamma\theta
    \end{bmatrix},&\theta\geq 0\\
    \frac{1}{25}\begin{bmatrix}
    1\;& 2\;& 2\;&4
    \end{bmatrix}
    \begin{bmatrix}
    1+\gamma\theta\\
    2+2\gamma\theta\\
    2+\gamma\theta\\
    4+2\gamma\theta
    \end{bmatrix},&\theta< 0
    \end{dcases}\\
    &=\begin{dcases}
    1+\frac{6}{5}\gamma\theta,&\theta\geq 0,\\
    1+\frac{3}{4}\gamma\theta,&\theta<0,
    \end{dcases}\\
    &=1+\frac{9}{10}\gamma\theta+\frac{3}{10}\gamma\theta(\mathbb{I}_{\{\theta\geq 0\}}-\mathbb{I}_{\{\theta< 0\}}).
\end{align*}

\subsection{Proof of Lemma \ref{le:truncation}}

Let $\{\nu(s,a)\}_{(s,a)\in\mathcal{S}\times\mathcal{A}}$ be any positive weights, and denote the weighted $\ell_p$-norm with weights $\{\nu(s,a)\}_{(s,a)\in\mathcal{S}\times\mathcal{A}}$ by $\|\cdot\|_{\nu,p}$. For any $x\in\mathbb{R}^{|\mathcal{S}||\mathcal{A}|}$, we have
\begin{align*}
    \min_{y\in B_r}\|x-y\|_{\nu,p}&=\min_{y\in B_r}\left(\sum_{s,a}\nu(s,a)|x(s,a)-y(s,a)|^p\right)^{1/p}\\
    &=\left(\sum_{s,a}\nu(s,a)\min_{-r\leq y(s,a)\leq r}|x(s,a)-y(s,a)|^p\right)^{1/p}\\
    &=\left(\sum_{s,a}\nu(s,a)|x(s,a)-\ceil{x(s,a)}|^p\right)^{1/p}\\
    &=\|x-\ceil{x}\|_{\nu,p}.
\end{align*}
Therefore, we have $\ceil{x}\in\arg\min_{y\in B_r}\|x-y\|_{\nu,p}$.

\section{Future Work}\label{ap:future}

\subsection{Establishing the Asymptotic Convergence and Improving the Function Approximation Error}
Although Theorem \ref{thm:main} establishes the mean-square error bound of $Q$-learning with linear function approximation, due to the function approximation error, the bound does not imply asymptotic convergence. In light of our discussion in Section \ref{sec:roadmap}, suppose Algorithm \ref{alg:main} indeed converges (as $K,T\rightarrow \infty$ and $\alpha\rightarrow 0$) . The corresponding $Q$-function estimate of the output, i.e., $\hat{Q}_T=\ceil{\Phi\hat{\theta}_T}$, can only converge to the solution of the \textit{truncated projected Bellman equation}:
\begin{align}\label{eq:TPBE}
    Q=\ceil{\text{Proj}_{\mathcal{W}}\mathcal{H}(Q)}.
\end{align}
Unlike projected Bellman equation (\ref{eq:PBE}), which may not have a solution in general (cf. Example \ref{example1}), since the truncated projected Bellman operator maps a compact set $B_r$ to itself, Eq. (\ref{eq:TPBE}) must have at least one solution according to the Brouwer fixed-point theorem. However, whether the solution to Eq. (\ref{eq:TPBE}) is unique or not is unclear. Therefore, it is also unclear if performing fixed-point iteration to solve Eq. (\ref{eq:TPBE}), or its stochastic variant (i.e., Algorithm \ref{alg:main}) can actually leads to asymptotic convergence. Further investigating the truncated projected Bellman equation to show asymptotic convergence is one of our immediate future directions.

Suppose we were able to show the asymptotic convergence of Algorithm \ref{alg:main} to the unique solution of the truncated projected Bellman equation (\ref{eq:TPBE}), denoted by $\bar{Q}$. Then, instead of establishing finite-sample bound of the form
\begin{align}\label{eq:40}
    \mathbb{E}[\|\hat{Q}_T-Q^*\|_\infty]\leq \underbrace{E_1+E_2+E_3}_{\text{go to zero as $K,T\rightarrow\infty$ and $\alpha\rightarrow 0$}}+\underbrace{E_4,}_{\text{Function approximation error}}
\end{align}
which is in fact what we did in this work,
we would seek to establish the finite-sample bound of $\mathbb{E}[\|\hat{Q}_T-\bar{Q}\|_\infty]$, and separately characterize the difference between $Q^*$ and $\bar{Q}$. This is in the same spirit of the seminal work \cite{tsitsiklis1997analysis}, which studies the TD-learning with linear function approximation algorithm for policy evaluation. There are two advantages of this alternative approach. One is that the sample complexity of $\hat{Q}_T$ converging to $\bar{Q}$ is well-defined once we establish finite-sample convergence of $\mathbb{E}[\|\hat{Q}_T-\bar{Q}\|_\infty]$ to zero, while the sample complexity of convergence bounds of the form (\ref{eq:40}) is strictly speaking not well-defined because of the additive constant $E_4$, and may lead to erroneous result, as illustrated in \cite{khodadadian2021finite2} Appendix C. Second, this approach would enable us to reduce the function approximation error by removing the $\sup$ operator in $\mathcal{E}_{\text{approx}}$, i.e.,  from the current $\sup_{Q:\|Q\|_\infty\leq r}\|\ceil{\text{Proj}_{\mathcal{W}}\mathcal{H}(Q)}-\mathcal{H}(Q)\|_\infty$ to  $\|\ceil{\text{Proj}_{\mathcal{W}}\mathcal{H}(\bar{Q})}-\mathcal{H}(\bar{Q})\|_\infty$.

Although the lack of asymptotic convergence is a major limitation of this work, we want to point out that such limitation is present in almost all related literature on both value-space and policy-space methods whenever function approximation is used. To our knowledge, the only exception is \cite{tsitsiklis1997analysis} (as well as its follow-up work), where asymptotic convergence was established for TD-learning, and the limit was characterized as the unique solution of the projected Bellman equation. Other literature studying RL with function approximation either do not have asymptotic convergence \cite{agarwal2021theory}, or have asymptotic convergence without knowing where the limit is \cite{maei2010toward}.

\subsection{The Deep $Q$ Network}\label{ap:DQN}
The ultimate goal of this line of work is to provide theoretical understanding to the celebrated Deep $Q$-Network. We first present the extension of our Algorithm \ref{alg:main} to the setting where we use arbitrary function approximation (cf. Algorithm \ref{alg:nonlinear}). Let $\mathcal{F}=\{f_\theta:\mathcal{S}\times\mathcal{A}\mapsto\mathbb{R}\mid \theta\in\mathbb{R}^d\}$ be a parametric function class (with parameter $\theta$). For example, $\mathcal{F}$ can be the set of functions representable by a certain neural network, and $\theta$ is the corresponding weight vector.

\begin{algorithm}[h]\caption{$Q$-Learning with Arbitrary Function Approximation}\label{alg:nonlinear}
\begin{algorithmic}[1] 
	\STATE {\bfseries Input:} Integers $T$, $K$, initialization $\theta_{0,0}=\hat{\theta}_0=\bm{0}$
	\FOR{$t=0,1,\cdots,T-1$}
	\FOR{$k=0,1,\cdots,K-1$}
	\STATE Sample $A_k\sim \pi_b(\cdot|S_k)$, observe $S_{k+1}\sim P_{A_k}(S_k,\cdot)$
	\STATE $\theta_{t,k+1}=\theta_{t,k}+\alpha_k \nabla f_{\theta_{t,k}}(S_k,A_k)(\mathcal{R}(S_k,A_k)+\gamma\max_{a'\in\mathcal{A}}\ceil{f_{\hat{\theta}_t}(S_{k+1},a')}-f_{\theta_{t,k}}(S_k,A_k))$
	\ENDFOR
	\STATE $\hat{\theta}_{t+1}=\theta_{t,K}$
	\STATE $S_0=S_K$
	\ENDFOR
	\STATE\textbf{Output:} $\hat{\theta}_T$
\end{algorithmic}
\end{algorithm}

While the algorithm easily extends, the theoretical results do not. In particular, there are two major challenges.

\begin{enumerate}[(1)]
    \item With recent advances in deep learning \cite{roberts2021principles}, it is possible to explicitly characterize the function approximation error $\mathcal{E}_{\text{approx}}$ as a function of the hyper-parameters of the chosen neural network, such as the width, the number of layers, and the H\"{o}lder continuity parameter, etc.
    \item A more significant challenge is about the convergence of the inner-loop of Algorithm \ref{alg:nonlinear}. Recall that in the linear function approximation setting, the inner loop (line 5 of Algorithm \ref{alg:main}) can be viewed as a one-step Markovian stochastic approximation for solving the linear system of equations $-\Phi^\top D\Phi \theta+\Phi^\top D\mathcal{H}(\ceil{\Phi\hat{\theta}_t})=0$, or a one-step Markovian stochastic gradient descent for minimizing a quadratic objective $\|\Phi \theta-\mathcal{H}(\ceil{\Phi\hat{\theta}_t})\|_D^2$ in terms of $\theta$. In this case, convergence to the global optimal of the inner-loop iterates is well established in the literature. Now consider using arbitrary function approximation in Algorithm \ref{alg:nonlinear}. Although the inner-loop (line 5) is still performing a one-step Markovian stochastic gradient descent for minimizing $\|f_\theta-\mathcal{H}(\ceil{f_{\hat{\theta}_t}})\|_D^2$ in terms of $\theta$, since the objective is now in general non-convex, the convergence to global optimal remains as a major theoretical open problem in the deep learning community.
\end{enumerate}

Although the Deep $Q$-Network was previously studied in \cite{fan2020theoretical}, their results rely on the following two assumptions: (1) the function approximation space is closed under the Bellman operator, and (2) there exists an oracle that returns the global optimal of non-convex optimization problems. Under these two assumptions, both challenges described
earlier are no longer present.

Once we explicitly characterize the function approximation error $\mathcal{E}_{\text{approx}}$, and show global convergence of the inner-loop, substituting the result into our analysis framework and we would be able to obtain finite-sample guarantees of Deep $Q$-Network, thereby achieving the ultimate goal of this line of research.

\section{Related Literature and Their Limitations}\label{ap:literature}

To complement Section \ref{subsec:literature}, we here present a more detailed discussion about related literature that requires strong negative drift assumptions, and that achieves stability of $Q$-learning with linear function approximation by implicitly changing the problem parameters.

\subsection{Strong Negative Drift Assumption}\label{ap:negative_drift}

As mentioned in Section \ref{subsec:literature}, classical $Q$-learning with linear function approximation (cf. Algorithm (\ref{alg:semi_gradient_QLFA})) was studied in \cite{melo2008analysis,chen2019finitesample, lee2019unified,xu2020finite,cai2019neural} and many other follow-up work under a strong negative drift assumption. While the specific assumption varies, they are all in the same spirit that the assumption should ensure that the associated ODE of $Q$-learning with linear function approximation is globally asymptotically stable, which essentially guarantees the stability of the algorithm \cite{borkar2009stochastic}. 

The negative drift assumptions are highly restrictive. To see this, we here present the assumption proposed in \cite{melo2008analysis} as an illustrative example:
\begin{align}\label{con}
    2\gamma^2\mathbb{E}_\mu[(\max_{a\in\mathcal{A}}\phi(S,a)^\top \theta)^2]<\mathbb{E}_{\mu,\pi_b}[(\phi(S,A)^\top \theta)^2],\;\forall\;\theta\neq \bm{0},
\end{align}
where the factor of $2$ is missing in \cite{melo2008analysis}. Since Condition (\ref{con}) needs to hold for all $\theta\neq \bm{0}$, it is not clear if it can be satisfied even if we choose the optimal policy as the behavior policy. Intuitively, to satisfy Condition (\ref{con}), the discount factor $\gamma$ should be extremely small.

To see more explicitly the restrictiveness of Condition (\ref{con}), we consider the case where $d=|\mathcal{S}||\mathcal{A}|$. A special case of this is when $\Phi$ is an identity matrix, which corresponds to the tabular setting. Since it is known that tabular $Q$-learning does not Condition (\ref{con}) to converge \citep{tsitsiklis1994asynchronous}, we would expect that Condition (\ref{con}) is automatically satisfied. However, the following result implies that Condition (\ref{con}) remains highly restrictive even when $d=|\mathcal{S}||\mathcal{A}|$.

\begin{lemma}\label{le:infeasible}
When $d=|\mathcal{S}||\mathcal{A}|$, then it is not possible to satisfy Condition (\ref{con}) when $\gamma\geq  \frac{1}{\sqrt{2|\mathcal{A}|}}$, where $|\mathcal{A}|$ is the size of the action-space.
\end{lemma}
\begin{proof}[Proof of Lemma \ref{le:infeasible}]
Lemma \ref{le:infeasible} is entirely similar to \cite{chen2019finitesample} Proposition 3. We here present its proof to make this paper self-contained. Define 
\begin{align*}
    \Theta_{s,a}=\text{span}\left(\left\{\phi(s',a')| (s',a')\in\mathcal{S}\times\mathcal{A},\;(s',a')\neq (s,a)\right\}\right)^{\perp},
\end{align*}
which is the orthogonal complement of the span of the feature vectors $\{\phi(s',a')\}_{(s',a')\neq (s,a)}$. Let $\theta \in \Theta_{s,a}$ satisfying $\phi(s,a)^\top \theta>0$ (which is always possible). Then Condition (\ref{con}) implies 
\begin{align*}
    2\gamma^2\mu(s)(\phi(s,a)^\top \theta)^2<\mu(s)\pi_b(a|s)(\phi(s,a)^\top \theta)^2,
\end{align*}
which implies $\gamma^2< \frac{\pi_b(a|s)}{2}$. Since this is true for all $(s,a)$, we must have $\gamma^2< \min_{(s,a)}\frac{\pi_b(a|s)}{2}\leq \frac{1}{2|\mathcal{A}|}$. Therefore, when $\gamma\geq \frac{1}{\sqrt{2|\mathcal{A}|}}$, it is not possible to satisfy Condition (\ref{con}).
\end{proof}

In this work, we do not require any variants of the negative drift assumption to achieve the stability of $Q$-learning with linear function approximation. Removing such strong assumption in existing literature is a major contribution of this work.

\subsection{A Variant of $Q$-Learning with Target Network \cite{zhang2021breaking}}\label{ap:TN}

A variant of the $Q$-learning with linear function approximation was proposed in \cite{zhang2021breaking}. To overcome the divergence issue, they introduced target network in the algorithm. However, as we have shown in Section \ref{subsec:insufficiency_of_TN}, target network alone is not enough to stabilize $Q$-learning. The reason that \cite{zhang2021breaking} achieves convergence is by implicitly modifying the problem discount factor. To see this, consider the tabular setting where $\phi_i$, $1\leq i\leq d$ are chosen as the canonical basis vectors. Then the algorithm proposed in \cite{zhang2021breaking} aims at solving the following modified Bellman equation (Eq. (11) in \cite{zhang2021breaking}):
\begin{align}\label{eq:169}
    (I+\eta D^{-1}) Q=\mathcal{H}(Q),
\end{align}
where $D$ is a diagonal matrix with the stationary distribution of the Markov chain $\{(S_k,A_k)\}$ (induced by the behavior policy) on its diagonal, and $\eta>0$ is a tunable parameter introduced in \cite{zhang2021breaking} to stabilize $Q$-learning.

Now note that as long as $\eta\neq 0$, Eq. (\ref{eq:169}) is not the same as the original Bellman equation $Q=\mathcal{H}(Q)$, which implies that the algorithm in \cite{zhang2021breaking} does not converge to $Q^*$ even in the tabular setting. We next show that introducing $\eta>0$ in Eq. (\ref{eq:169}) is equivalent to artificially scaling down the discount factor $\gamma$ of the problem.

For ease of exposition, suppose that we are in the ideal setting where $D=\frac{1}{|\mathcal{S}||\mathcal{A}|}I$ (i.e., uniform exploration). Then Eq. (\ref{eq:169}) can be equivalently written by
\begin{align}\label{eq:225}
    Q(s,a)=\frac{1}{1+\eta|\mathcal{S}||\mathcal{A}|} \mathcal{R}(s,a)+\frac{\gamma}{1+\eta|\mathcal{S}||\mathcal{A}|}  \mathbb{E}[\max_{a'\in\mathcal{A}}Q(S_{k+1},a')\mid S_k=s,A_k=a],\quad \forall\;(s,a).
\end{align}
Compared to the original Bellman equation $Q=\mathcal{H}(Q)$, the modified Bellman equation (\ref{eq:225}) has two major modifications. First is that the reward function is scaled down by a factor of $1+\eta|\mathcal{S}||\mathcal{A}|$. This modifications does not change the optimal policy since the optimal policy is invariant to the scaling of the optimal $Q$-function. A more important modification is that the discount factor $\gamma$ is scaled down by a factor of $1+\eta|\mathcal{S}||\mathcal{A}|$. This change potentially results in a different optimal policy compared to the original problem. In fact, since the tunable parameter $\eta$ is first multiplied by $|\mathcal{S}||\mathcal{A}|$ before it appears in the denominator of $\gamma$ in Eq. (\ref{eq:225}), using postive $\eta$ changes the discount factor of the problem drastically.

\subsection{Coupled $Q$-Learning \cite{carvalho2020new}}\label{ap:Melo}
A two time-scale variant of $Q$-learning with linear function approximation (called coupled $Q$-learning) was proposed in \cite{carvalho2020new}.
It was shown in \cite{carvalho2020new} that the limit points $u^*$ and $v^*$ of the coupled $Q$-learning algorithm satisfy the following systems of equations:
\begin{align}
    u^*&=\Phi^\top D\mathcal{H}(\Phi u^*)\label{eq:325}\\
    v^*&=(\Phi^\top D\Phi)^{-1} u^*,\label{eq:326}
\end{align}
where $D$ is a diagonal matrix with diagonal entries being the stationary distribution of the Markov chain $\{(S_k,A_k)\}$ induced by the behavior policy $\pi_b$.

Under the assumption that $\|\phi(s,a)\|_2\leq 1$ for all $(s,a)$ (Assumption 2 in \cite{carvalho2020new}), and $\Phi^\top D\Phi=\sigma I_d$ (Assumption 4 in \cite{carvalho2020new}), where $d$ is the number of basis vectors, a performance guarantee is provided regarding the distance between the $Q$-function estimate $Q_{v^*}$ associated with $v^*$ and the optimal $Q$-function $Q^*$, and is presented in the following.

\begin{theorem}[Theorem 2 in \cite{carvalho2020new}]\label{thm:Melo}
The limit point $v^*$ satisfies
\begin{align}\label{eq:289}
    \|Q_{v^*}-Q^*\|_\infty\leq \frac{1}{1-\gamma}\|Q^*-\text{Proj}_{\mathcal{W}}(Q^*)\|_\infty+\mathcal{E}_\sigma,
\end{align}
where $\mathcal{E}_\sigma=\frac{1-\sigma}{\sigma}\frac{\gamma}{(1-\gamma)^2}$.
\end{theorem}

Note that in Eq. (\ref{eq:289}) of Theorem \ref{thm:Melo}, in addition to the function approximation error, there is an additional error term $\mathcal{E}_\sigma$ that does not vanish even in the tabular setting. Although the coupled $Q$-learning algorithm does not require strong assumptions to converge, we next show that in order for the performance bound (\ref{eq:289}) to be non-trivial, the discount factor $\gamma$ must be sufficiently small.

Consider the error term $\mathcal{E}_\sigma=\frac{1-\sigma}{\sigma}\frac{\gamma}{1-\gamma}\frac{1}{1-\gamma}$.
Since $\|Q^*\|_\infty\leq \frac{1}{1-\gamma}$, in order for the performance bound of Theorem \ref{thm:Melo} to be meaningful, we need to at least have
\begin{align*}
    \frac{1-\sigma}{\sigma}\frac{\gamma}{1-\gamma}<1,
\end{align*}
otherwise simply choosing $Q=\bm{0}$ leads to a better performance guarantee. The above inequality implies $\gamma<\sigma$.
However, under the assumption that $\|\phi(s,a)\|_2\leq 1$ for all $(s,a)$ (Assumption 2 in \cite{carvalho2020new}), and $\Phi^\top D\Phi=\sigma I_d$ (Assumption 4 in \cite{carvalho2020new}), we have
\begin{align*}
    d\sigma=\sum_{i=1}^d\sum_{s,a}\phi_i(s,a)^2D(s,a)
    =\sum_{s,a}D(s,a)\sum_{i=1}^d\phi_i(s,a)^2
    \leq \sum_{s,a}D(s,a)
    =1,
\end{align*}
which implies $\sigma\leq \frac{1}{d}$. As a result, Theorem \ref{thm:Melo} provides a meaning performance guarantee on the limit point $v^*$ only when $\gamma\leq \frac{1}{d}$, which is a restrictive requirement on the discount factor $\gamma$ of the problem. 

To see more explicitly the reason that the coupled $Q$-learning algorithm has an additional bias $\mathcal{E}_\sigma$,  consider the tabular setting, i.e., $\Phi=I_{|\mathcal{S}||\mathcal{A}|}$. In this case Assumption 4 of \cite{carvalho2020new} reduces to $D=\frac{1}{|\mathcal{S}||\mathcal{A}|}I_{|\mathcal{S}||\mathcal{A}|}$ (uniform exploration). Then Eq. (\ref{eq:325}) is equivalent to
\begin{align*}
    Q_{u^*}(s,a)&=\frac{1}{|\mathcal{S}||\mathcal{A}|}[\mathcal{H}(Q_{u^*})](s,a)\\
    &=\frac{1}{|\mathcal{S}||\mathcal{A}|}\mathcal{R}(s,a)+\frac{\gamma}{|\mathcal{S}||\mathcal{A}|}\mathbb{E}[\max_{a'\in\mathcal{A}}Q_{u^*}(S_{k+1},a')\mid S_k=s,A_k=a],\quad \forall\;(s,a).
\end{align*}
As illustrated in the previous subsection, such modification of the Bellman equation is equivalent to artificially scaling down the discount factor $\gamma$ of the original problem by a factor of $|\mathcal{S}||\mathcal{A}|$. In view of Eq. (\ref{eq:326}), when $\Phi=I_{|\mathcal{S}||\mathcal{A}|}$ and $D=\frac{1}{|\mathcal{S}||\mathcal{A}|}I_{|\mathcal{S}||\mathcal{A}|}$, $Q_{v^*}$ is just a constant scaling of $Q_{u^*}$. Hence the optimal policy induced from either $Q_{v^*}$ or $Q_{u^*}$ is the one that corresponds to the original problem with the discount factor $\gamma$ being replaced by $\frac{\gamma}{|\mathcal{S}||\mathcal{A}|}$. Because of this implicit modification on the problem discount factor, the coupled $Q$-learning algorithm although is stable, does not converge to $Q^*$ even in the tabular setting.
\end{appendix}

\end{document}